%% file: main.tex
\documentclass[11pt]{article}
\def\comments{0}

\usepackage{preamble}
\usepackage[hang,flushmargin]{footmisc}

\newcommand{\nup}{n_{\mathrm{up}}}
\newcommand{\Atrain}{A_{\mathrm{train}}}
\newcommand{\Aup}{A_{\mathrm{up}}}
\newcommand{\INrm}{\mathrm{IN}}
\newcommand{\OUTrm}{\mathrm{OUT}}
\newcommand{\uprm}{\mathrm{up}}

\newcommand{\restrm}{\mathrm{rest}}
\newcommand{\sgdfull}{SGD-Full\xspace}
\newcommand{\sgdnew}{SGD-New\xspace}
\newcommand{\diff}{$\textsc{ScoreDiff}$\xspace}
\newcommand{\ratio}{$\textsc{ScoreRatio}$\xspace}

\renewcommand{\epsilon}{\varepsilon}

\ifnum\comments=1
    \newcommand{\matthewmargin}[1]{\marginpar{\tiny\sf\color{orange}{MJ: #1}}}
    \newcommand{\matthew}[1]{\matthewmargin{#1}}
    \newcommand{\jon}[1]{\marginpar{\tiny\sf\color{blue}{JU: #1}}}
    \newcommand{\alina}[1]{\marginpar{\tiny\sf\color{red}{AO: #1}}}

    \newcommand{\roxana}[1]{\roxanamargin{#1}}
\else
    \newcommand{\matthew}[1]{}
    \newcommand{\matthewmargin}[1]{}
    \newcommand{\jon}[1]{}
    \newcommand{\alina}[1]{}
    
    \newcommand{\roxana}{}
\fi

\title{How to Combine Membership-Inference Attacks on \\Multiple Updated Models\thanks{Authors ordered by contribution.  MJ and SW made equal contributions.}}

\author{
Matthew Jagielski\thanks{Google, Inc.\ and Khoury College of Computer Sciences, Northeastern University}\hspace{50pt}
\and Stanley Wu\thanks{ Khoury College of Computer Sciences, Northeastern University.  Supported by NSF award CCF-1750640.}\hspace{50pt}
\and Alina Oprea\thanks{Khoury College of Computer Sciences, Northeastern University. Supported by NSF award CNS-2120603 and a gift from Apple.}
\and Jonathan Ullman\thanks{Khoury College of Computer Sciences, Northeastern University.  Supported by NSF awards CCF-1750640 and CNS-2120603 and a gift from Apple.}\hspace{50pt}
\and Roxana Geambasu\thanks{Department of Computer Science, Columbia University. Supported by NSF award CNS-2120603, U.S.\ Department of Energy award DE-SC-0001234, and fellowships from Microsoft and the Sloan Foundation.}
}

\date{\vspace{-5mm}}

\begin{document}

\maketitle	
\begin{abstract}
  A large body of research has shown that machine learning models are vulnerable to membership inference (MI) attacks that violate the privacy of the participants in the training data.  Most MI research focuses on the case of a single standalone model, while production machine-learning platforms often update models over time, on data that often shifts in distribution, giving the attacker more information.  This paper proposes new attacks that take advantage of one or more model updates to improve MI.  A key part of our approach is to leverage rich information from standalone MI attacks mounted separately against the original and updated models, and to combine this information in specific ways to improve attack effectiveness.  We propose a set of combination functions and tuning methods for each, and present both analytical and quantitative justification for various options.  Our results on four public datasets show that our attacks are effective at using update information to give the adversary a significant advantage over attacks on standalone models, but also compared to a prior MI attack that takes advantage of model updates in a related machine-unlearning setting.  We perform the first measurements of the impact of distribution shift on MI attacks with model updates, and show that a more drastic distribution shift results in significantly higher MI risk than a gradual shift. Our code is available on \href{https://github.com/stanleykywu/model-updates}{GitHub}.
\end{abstract}

\input{intro}

\input{related}
\input{background}
\input{attack-algorithms}

\input{evaluation-fewer-questions}
\input{disc}

\bibliographystyle{plain}
\bibliography{refs}

\appendix
\input{appendices}

\end{document}

%% file: intro.tex
\section{Introduction}
\label{sec:introduction}

Machine learning models are often trained on sensitive user data, and it is crucial to ensure that these models respect the privacy of individuals who contribute their data. However, there is now a robust body of literature demonstrating that unless explicit steps are taken to ensure privacy, these models will leak sensitive information.  In particular, as first shown more than a decade ago by Homer et al.~\cite{homer2008resolving}, even the simplest statistical models, when trained via standard techniques, will allow for \emph{membership-inference (MI) attacks}, in which an attacker can detect the presence of individuals in the training data, or in certain subsets of the training data.  MI attacks can be a concerning privacy violation on their own, if membership in the training data, or a specific subset of the training data, indicates something sensitive about the user or can be used as a step towards reconstructing training examples.  MI attacks have since become an active area of research in statistics, machine learning, and security, and we now have a rich toolkit of MI attacks~\cite{SankararamanOJH09, BunUV14, dwork2015robust}, which notably includes black-box attacks on modern large models in supervised-learning~\cite{shokri2017membership, yeom2018privacy, jagielski2020auditing, NasrSTPC21, carlini2021membership}.\jon{Any important MI cites missing?}  

Most MI research has focused on the case of a single standalone model.  However, in real machine learning workloads, models are typically \emph{updated} as new training data arrives, and attackers have the ability to observe some aspects of the model both before and after updates.  Intuitively, giving the attacker the ability to see the model before and after the update should reveal  information about the specific training examples in the update set.  To illustrate, consider the effect of model updates on MI for the simple statistical task of {\em mean estimation}.  
\begin{example} \label{ex:mean-est}
We are given a set of training examples $x_1,\dots,x_n \in \mathbb{R}^d$ and want to release its mean.  The theory of MI attacks~\cite{SankararamanOJH09, dwork2015robust} tells us that, under natural conditions on the data, we will be able to accurately infer membership if and only if $d \gg n$.  However, suppose that data arrives over time and we initially release the mean $\mu_0$ of the first $n-t$ points for some $t \ll n$, followed by an update $\mu_1$ with the mean of all $n$ points.  It is not hard to see that we can combine $\mu_0$ and $\mu_1$ to obtain the mean of just the last $t$ points, and we can perform accurate membership inference on this even when $n \gg d \gg t$, which is impossible given only $\mu_1$.
\end{example}

While the statistical models trained over user data by modern machine-learning workloads are far more complex than a simple mean, the same kind of effect is expected from repeated releases of updated models.  Such repeated releases are common.  Production training platforms, such as Amazon SageMaker~\cite{sagemaker}, Azure Machine Learning~\cite{azure}, and Tensorflow-Extended~\cite{tfx,46484}, all support automatic model updating on newly collected data to keep up with changing distributions or improve models over time. The frequency and method for model updates differs, but often an update is triggered by the arrival of a sizeable data batch, such as a day's or a week's worth of data, and involves fine-tuning an already deployed model on data from the new batch and potentially samples from previous batches~\cite{46484}. For example, a news-recommendation model may be updated daily to keep up with events in the news, and a product-recommendation model may be updated weekly to capture evolving trends.  In each case, updated models are repeatedly released for serving or are pushed to users' mobile devices or servers all around the world for faster predictions.

This paper investigates the threat of repeated model updates for an attacker who monitors their releases and wishes to infer membership of specific samples in each update dataset. We formalize the problem of membership-inference under repeated model updates in a way that supports a wide range of model update procedures, sizes of update batches, and distribution shift in the new data (Section~\ref{sec:mi_updates}). Geared toward this problem, we develop the first black-box MI attack algorithms that combine information from previously known standalone MI attacks---such as the state-of-the-art LiRA attack~\cite{carlini2021membership}---to let the adversary take advantage of access to both the original model and one or more updated models to improve MI on the update set (Section~\ref{sec:attack_algorithms}).  Our algorithms compute the standalone attack's \emph{confidence scores} separately against the original model, then against the update model(s), and combine them to obtain a confidence score for membership in the update set.  We justify the need to use detailed confidence scores information by showing that combining only the binary membership decisions does not increase the attacker's power.  We consider two different methods for combining scores, each motivated analytically by the study of a simple example.  Our analysis and experiments demonstrate that the best choice of score will depend on the specific learning algorithm being attacked. 

Some previous works have studied MI attacks involving multiple models trained on overlapping datasets~\cite{salem2020updates, ZanellaBeguilinWTRPOKB20}, for example arising from intermediate computations revealed by federated learning systems~\cite{wang2019eavesdrop, 8835269,nasr2019comprehensive} or from model unlearning~\cite{chen2020machine}.  Our work, however, is the first to study a number of aspects specific to the repeated model-update setting, including updates with sizeable batches of data, multiple updates, and the effect of distribution shift.

We evaluate our algorithms on four datasets---FMNIST, CIFAR-10, Purchase100, and IMDb---using suitable linear and neural network models.  We highlight several key contributions and conclusions:
\begin{enumerate}
\item We demonstrate that access to one or more updated models makes an attacker significantly more effective at inferring membership in the training data, compared to having only a single standalone model (e.g., on FMNIST, MI accuracy increases from 52\% without updates to 79\% with updates).  We show that this effect is more dramatic the more distribution shift occurs between the original training and the update.

\item We consider a variety of attack algorithms and tuning methods, and demonstrate both analytically and empirically that no single method is best in all situations, highlighting the need for a varied testing strategy.  We also demonstrate that our attacks are more efficient and effective than attacks designed for the related, but distinct, setting of machine unlearning~\cite{chen2020machine}.

\item We consider multiple algorithms for updating models, and show that for small update sets, using the whole dataset to update the model is less vulnerable to MI attacks compared to training on just the update set, and the opposite is true for larger update sets.  Our findings offer some guidance for practitioners employing basic defenses.

\item We audit differentially private defenses, and show that they offer strong protection with small privacy parameter, but only modest protection with large parameter.  In particular, the worst-case bounds offered by differential privacy can be close to tight in practice (within 2.0-3.6x in some settings).

\end{enumerate}

Overall, our work demonstrates that the model updates arising from production machine-learning systems significantly increase the risk of membership-inference attacks, and highlights the many subtleties that arise in constructing both attacks and defenses.

%% file: related.tex
\section{Related Work} \label{sec:related}


A long line of work has recognized the privacy risks of machine learning. Early work considered very simple statistical tasks~\cite{homer2008resolving, SankararamanOJH09, BunUV14, dwork2015robust}, and black-box attacks on machine learning algorithms have were developed later~\cite{yeom2018privacy, shokri2017membership, song2021systematic}. Other types of privacy attacks have also been considered, such as training data reconstruction/extraction~\cite{salem2020updates, carlini2020extracting}, attribute inference~\cite{yeom2018privacy}. MI attacks have also been considered in white-box~\cite{leino2020stolen} and label-only settings~\cite{choo2020label, li2020membership}, and have been evaluated in unbalanced scenarios, where training points are much less common than test points~\cite{jayaraman2021revisiting}.

\mypar{Membership-Inference with Model Updates.} Our work considers an adversary seeking to run MI attacks in the setting where a model is updated repeatedly over its lifetime. A few privacy attacks on model updates have been considered in prior work, but in distinct settings compared to ours.

First, in federated learning, a decentralized model update procedure, attacks have been demonstrated to leak the label composition~\cite{wang2019eavesdrop}, attributes present in local datasets~\cite{8835269}, and MI attacks~\cite{nasr2019comprehensive}. 

Second, Salem et al.~\cite{salem2020updates} show that an attacker with query access to an initial and an updated model can perform {\em reconstruction attacks} to recover the labels and feature values of the update points. The reconstruction attack is designed for small updates to the model, and works only in the online learning setting where only new points are considered when generating the update. The differences from our work are: (1) we focus on {\em membership inference} attacks on models supporting multiple updates; (2) we develop new attacks for this setting that combine existing standalone-model MI attacks without the heavyweight construction of many shadow models; and (3) we consider a range of model update regimes, by varying the size of the update samples, the distribution shift in the new data, and the retraining procedure, including attacks for multiple updates. 

Third, Chen et al.~\cite{chen2020machine} construct MI attacks for machine unlearning, in a setting where an adversary has access to an initial model and an updated model after a set of examples are removed from training. We observe that, when updating or unlearning are performed by retraining from scratch, an adversary with access to the models before and after unlearning is related to a model update adversary having access to the models before and after updating. We perform an experimental comparison with their attacks in Section~\ref{sec:eval:q7} and find that our attacks, geared towards our specific setting, are significantly more powerful compared to theirs, which were not designed for this setting. In addition, the model updates setting considered by our work is more general in supporting multiple updates, data distribution shift under updates, and different training regimes (online learning and full retraining). 

\mypar{Memorization Attacks.} Memorization attacks against generative language models demonstrate that training data can be extracted by an adversary with black-box query access to a model~\cite{carlini2019secret, carlini2020extracting}. \cite{carlini2020extracting} generate and rank text samples from GPT-2 and use MI attacks to test that a generated sample belongs to the training data. \cite{zanella2020analyzing} show that model updates in generative language models improve memorization attacks.

\mypar{Differentially Private ML.} Privacy attacks have inspired privacy-preserving training algorithms, including defenses specifically designed to prevent MI attacks~\cite{nasr2018machine, jia2019memguard}, as well as the adoption of differential privacy~\cite{dwork2006calibrating}. Differentially private machine learning algorithms~\cite{bassily2014private, song2013stochastic, abadi2016deep, chaudhuri2011differentially} are a defense against MI attacks, and while deployments of them exist, they are still relatively rare compared to the scale of machine learning workloads at large companies.  Still, we perform an experimental evaluation of differential privacy in the context of our attacks and show that differential privacy is an effective protection at low privacy parameters, but also that our attacks can be an effective empirical audit of a differential privacy deployment.

%% file: background.tex
\section{MI with Model Updates}
\label{sec:mi_updates}


\subsection{Background}

Many supervised learning algorithms exist to train machine learning models from labeled data.  In this work, we consider classification problems, where samples are taken from a data domain  $\mathcal{X}$ and the output space is a discrete set of $K$ classes $\mathcal{Y}=[K]$.\footnote{Here and throughout we use the notational shorthand $[K] = \{1,2,\dots,K\}$.} With a training algorithm, the learner typically learns some set of parameters $\theta$ which are used to evaluate the model function $f$ and minimize a loss function. 



In our empirical evaluation, we consider  logistic regression and various neural network models. A neural network is a model $f(x)$ which is computed as the chain of $L$ layers $g_L \circ g_{L-1} \circ \cdots g_1 \circ x$, where each layer function $g_i$ takes its input $a_i$ and computes $h(W_i a_i + b_i)$, where $W_i, b_i$ are trainable weights (the parameters $\theta$ are the weights from every layer) and $h$ is a nonlinear activation function. A common activation function is the ReLU $h(x)=\max(0, x)$~\cite{glorot2011deep}. In tabular data, it is common to use simple models, such as logistic regression (where $L=1$) or small neural networks (small $L$). Image data typically uses deeper networks which use \emph{convolutions}~\cite{lecun1998gradient}, a constraint on the weight matrices which exploits the structure of images to significantly improve performance. For classification, networks typically use a softmax output, which produces probability values for each of the $K$ classes.

Training models typically proceeds by using gradient descent on a given dataset $D$. This requires defining a loss function $\ell$, which measures the model's performance on a given data point $x, y$. After initializing $\theta$, a batch of $B$ samples $$\lbrace x_j, y_j\rbrace_{j=1}^B$$ is selected from the training dataset, and the model parameters are updated in the inverse gradient direction of the loss, averaged over the batch set, as $$\theta' = \theta - \tfrac{\eta}{B}\sum_{j=1}^B\nabla_{\theta}\ell(f(x; \theta), y),$$
where $\eta$ is a learning rate. Batches are sampled from $D$ until every point has been used, and this is repeated multiple times, called epochs.

\subsection{Threat Model}


We consider supervised machine learning problems in which the model parameters are updated over time by retraining with new data, which is common in applications for a variety of reasons.  In general, as more labeled data becomes available, models should be updated to correct for potential errors and improve their generalization.  Another key reason for model update is the potential shift in data distribution over time. For example, a sentiment analysis model trained on news article needs to be adapted to take into account  recent events, and models used for financial market forecasts need to be updated as the market evolves continuously. The frequency of model updates is dependent on the application requirements, the data distribution shift, and the envisioned deployment scenarios. For many industrial applications,  model retraining has become part of machine learning deployment pipelines. Our goal is to study the privacy implications of model updates over time, considering factors such as the size of the update, the number of updates, the data distribution shift, and the training regime. 

\begin{wrapfigure}{R}{.5\textwidth}
    \centering
    \includegraphics[width=0.9\linewidth]{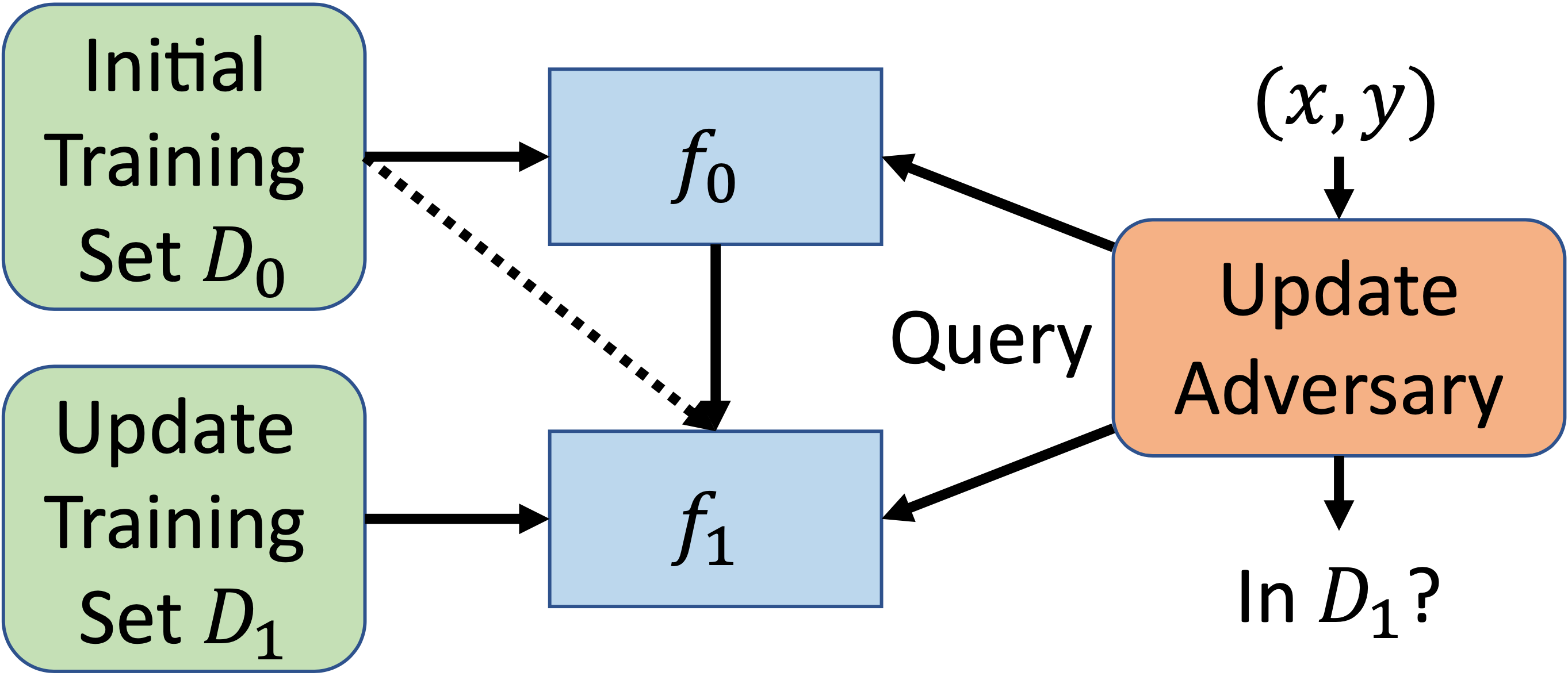}
    \caption{Membership-inference with a single model update.  An adversary with query access to both models $f_0$ and $f_1$ can more effectively distinguish whether a query point $(x, y)$ is in the update set $D_1$ or is an independent sample from the same distribution as $D_1$.}
    \label{fig:diagram}
\end{wrapfigure}

In our setting, a {\em learner} is given $k+1$ datasets $D_0$, $D_1$, $\cdots$ $D_k$ over time, where each dataset consists of $n_i$ samples $$D_i=
\left\{ \left( x_j^{(i)}, y_j^{(i)}\right) \right\}_{j=1}^{n_i}$$ selected from a distribution $\mathcal{D}_i$. The learner runs a training algorithm $\Atrain$ on $D_0$ to produce an initial model $f_0$. The learner is then provided with each new dataset $D_i$, and produces a new model $f_i$ by running an update algorithm $\Aup$ using only datasets $D_0$, $D_1$, $\dots$, $D_i$. This process is described in Algorithm~\ref{alg:model_updates}. In our work, we consider a fixed value for the update size at each iteration $n_i=\nup$ for all $i>0$, and vary this value $\nup$.

\begin{algorithm}
\KwData{Base learning algorithm $\Atrain$, Update algorithm $\Aup$, number of updates $k$, list of datasets $[D_0, D_1, \dots, D_k]$}

$f_0 = \Atrain(D_0)$

\lFor{$i\in 1 \dots k$:}{
  $f_i = \Aup(f_{i-1}; D_i, \cdots, D_1, D_0)$
}\Return $f_0, f_1, \cdots, f_k$
\caption{Generating Model Updates}
\label{alg:model_updates}
\end{algorithm}

In standard membership inference (MI) attacks~\cite{shokri2017membership, yeom2018privacy, salem2020updates}, the {\em adversary} can interact with the machine learning model in a black-box manner, with the goal of distinguishing if a data sample was part of the training set or not. In our model update setting, we 
consider a black-box adversary $\mathcal{A}$, who is capable of observing the output of each model $f_i$ on multiple query points, but does not have knowledge of the specific models architecture or parameters. As models are retrained with new data, the adversary's goal is to infer if a data sample was part of the update set or not. In the setting with multiple model updates, the adversary is also interested in inferring at which time epoch the data sample was used to update the model. Figure~\ref{fig:diagram} visualizes this threat model.



\subsection{Formalization of MI under Model Updates}

We formalize the problem of membership inference with model updates by adapting the membership inference experiment of \cite{yeom2018privacy} to the model update setting. We present the experiment in full generality, and introduce specific contexts for which we subsequently develop specific attack algorithms.

\paragraph{The membership-inference experiment.} 
Let $\Aup$ be a model update algorithm, $D_0$ an initial training set, $f_0$ an initial model, $\nup$ an update set size, and each $\mathcal{D}_i$ with $i\in [k]$ a distribution over samples $(x, y)$.  We define the membership-inference experiment in Algorithm~\ref{fig:mi-experiment}. The experiment allows the adversary access to all updated models, and requires it to distinguish between update and test data.

\begin{algorithm}[h]
	\DontPrintSemicolon	
	\SetKwFunction{ExpMIUpdates}{ExpMIUpdates}
	\SetKwProg{Fn}{Function}{:}{}
	\Fn{\ExpMIUpdates{$\mathcal{A}, \Aup, D_0, f_0, \nup, \mathcal{D}_1, \mathcal{D}_2, \cdots, \mathcal{D}_k$}}{
		\begin{enumerate}[nosep]
			\item For $i\in [k]$, draw $\nup$ samples from $\mathcal{D}_i$: $D_{\mathrm{up}}^i\sim \mathcal{D}_i^{\nup}$.
			\item Generate a sequence of models $f_1, f_2, \cdots, f_k$ by iterating $f_i = \Aup(f_{i-1}; D_i,\dots,D_1,D_0)$ as in Algorithm~\ref{alg:model_updates}.
			\item Sample $a \in [k]$ and $b \in \{0,1\}$ uniformly at random. $b$ represents whether the test sample \\ is in training or not, while $a$ represents which update it belongs to.
			\item If $b = 0$, sample $(x, y)\sim \mathcal{D}$, otherwise $(x, y)\sim D_{\mathrm{up}}^a$.
			\item Let $(\hat{a},\hat{b})=\mathcal{A}(x, y; f_0, f_1,\dots,f_k)$ be the output of the attacker.  Let $s_{a}$ indicate if $\hat{a} = a$ let $s_{b}$ \\ indicate if $\hat{b} = b$, and return $(s_a,s_b)$.
		\end{enumerate}
	}
	\caption{The membership-inference experiment with model updates}
	\label{fig:mi-experiment}
\end{algorithm}


We instantiate this experiment in three settings, for attack algorithm development (Section~\ref{sec:attack_algorithms}) and evaluation (Section~\ref{sec:evaluation}): single update, multiple updates, and single update distribution shift.

\paragraph{Single model update.} In the single update setting, we consider $k=1$, $\mathcal{D}_1=\mathcal{D}$, so the initial training distribution does not change to sample updates. This setting lets us understand the difference between having access to the models before and after an update compared to only the final model.
To evaluate the performance of the attacks, we can measure {\em accuracy}, $\mathbb{E}[s_0]$, and {\em precision}, $\mathbb{E}[s_0|g_0=1]$. An attack maximizing precision may differ from one maximizing accuracy. We also measure {\em recall}, $\mathbb{E}[g_0=1|b=1]$, but, as noted in \cite{leino2020stolen} and \cite{jayaraman2021revisiting}, a membership inference attack achieving high precision is likely to be more harmful than one achieving high recall.\footnote{Carlini et al. \cite{carlini2021membership} suggests measuring true positive rate (TPR) and false positive rate (FPR). TPR is identical to recall, and FPR can be computed easily from precision, given that $b$ takes the value 0 and 1 equally often.} An attack which classifies every sample as appearing in the training set is not harmful, but obtains a high recall; meanwhile, an attack correctly identifying a single sample as appearing in the training set achieves tiny recall, but is harmful to that sample.

\paragraph{Multiple model updates.}
The multiple update setting considers $k>1$ and $\mathcal{D}_i=\mathcal{D}$ for all $i$, so the training distribution remains constant. This lets us understand the difference between access to multiple models and only the last model.
There is a richer set of metrics that can be used here to evaluate attack performance. We measure the \emph{generic accuracy} as the success of the attack at inferring if a data sample is part of any of the $k$ update datasets, $\mathbb{E}[s_0]$. We also measure the accuracy with which the precise update epoch is identified, $\mathbb{E}[s_1]$. We call this \emph{specific accuracy}. Note that any attack for generic accuracy can be converted to a specific accuracy attack by randomly selecting one of the $k$ update datasets.

\paragraph{Distribution shift.}
In this setting, we consider $k=1$ and $\mathcal{D}_1\neq \mathcal{D}$. While distribution shift may happen over several updates, we elect to isolate the impact that distribution shift has on our attacks.  To measure attack performance, we use the same metrics as in the single update setting: accuracy, precision, and recall.

\paragraph{Retraining methods.} There are multiple ways to use the datasets $D_0, D_1, \cdots D_i$ to update model $f_{i-1}$ at epoch $i$, which we call:
\begin{itemize}
    \item \textbf{\sgdnew.} This strategy updates the model $f_{i-1}$ using only the new training set $D=D_i$.  To prevent forgetting earlier datasets, one must use a small learning rate and few epochs.
    \item \textbf{\sgdfull.} This strategy updates the model $f_{i-1}$ using the entire training set available at epoch $i$, $D=[D_0, D_1, \cdots D_i]$.  With a larger dataset, one can increase the learning rate and number of epochs, at the cost of using less recent data.
\end{itemize}

These training strategies have also been used by previous work considering model updates. Zanella-B\'{e}guelin et al.~\cite{zanella2020analyzing} compare both strategies for extracting training data from generative language models, while Salem et al.~\cite{salem2020updates} use a variant of continual training for mounting reconstruction attacks.

%% file: attack-algorithms.tex
\section{Attack Algorithms}
\label{sec:attack_algorithms}
We develop attacks for the single and multiple update instantiations of the MI under model update problem introduced in the preceding section.  Our attacks are generic with respect to distribution shift and retraining methods, so we discuss those topics directly as part of our evaluation of the proposed attacks (Section~\ref{sec:evaluation}).
We first focus on attacks for the single update setting (Section~\ref{sec:one_update}).
We propose multiple options and subsequently justify analytically both their designs and the need for the options (Section~\ref{sec:algos-justification}).
Finally, we propose several attack options for the multiple update setting (Section~\ref{sec:mult_update}).
Table~\ref{tab:attacks}, placed at the end of this section, summarizes the various attacks and their options for easy access.

\subsection{Single Update Attacks}
\label{sec:one_update}

Given a single model $f_0$ trained on a dataset $D_0$, and an individual target example $(x,y)$, the standard black-box way to test membership of $x$ in $D_0$ is to compute an appropriate {\em score function} $\ell(x,y; f_0)$ and apply some {\em threshold} to this score.  

Now suppose we are given two models $f_0,f_1$ trained on datasets $D_0$ and $D_0 \cup D_1$ respectively, and a target example $(x,y)$, whose membership in $D_1$ we want to infer.  Intuitively, being a member of $D_1$ is equivalent to being a member of $D_0 \cup D_1$ and a non-member of $D_1$.  So, a first attempt is simply to infer membership in $D_0 \cup D_1$ and membership in $D_0$ and decide membership in $D_1$ appropriately.  However, as we show in Section~\ref{sec:algos-justification}, if the score information is binary (e.g.\ it is the output of some membership-inference attack for each standalone model), then model updates \emph{do not} increase the accuracy of membership-inference attacks.

However, score-based membership-inference attacks give more information than just the binary outcome, and a key contribution of our work is to show how to strictly outperform the preceding baseline by {\em combining} the two scores, $\ell(x,y;f_0)$ and $\ell(x,y;f_1)$.  Given these two scores, there are multiple logical ways we could combine them to produce a single score for membership in $D_1$.  In this work, we introduce two main strategies: $\textsc{ScoreDiff}$ and $\textsc{ScoreRatio}$.  As their names suggest, we define
\begin{align*}
&\textsc{ScoreDiff}(x, y, f_0, f_1, \ell) = \ell(x,y,f_1) - \ell(x,y;f_0) \textrm{ and} \\
&\textsc{ScoreRatio}(x, y, f_0, f_1, \ell) = \frac{\ell(x,y,f_1) + c}{\ell(x,y;f_0) + c}
\end{align*}
where $c > 0$ is a damping constant to avoid instability when the denominator is close to $0$. As we will show empirically in Section~\ref{sec:evaluation}, neither of these two methods for combining scores dominates the other, and in Section~\ref{sec:algos-justification} we give analytical justification for why each score can sometimes be superior. 

We view the use of these particular combiners as a second key contribution of our work. While we could take the approach of~\cite{chen2020machine} and learn how best to combine scores from scratch, our evaluation will show that choosing a fixed combiner, such as \diff or \ratio, is both more efficient and more effective.

To use these methods for combining scores, we need to do two things: (1) instantiate these strategies by choosing the score function, and (2) determine a threshold to apply to convert the real-valued scores into a binary membership decision.  As score functions, in this paper we use the standard cross-entropy loss~\cite{yeom2018privacy} and the state-of-the-art LiRA score function~\cite{carlini2021membership}. These are both computable with only class probabilities, but LiRA requires training shadow models. Future single-model MI attacks might provide even better score functions that our strategies can incorporate.

In this work we consider multiple ways to set the threshold.  To motivate these methods, we return to the standard interpretation of membership inference as a hypothesis-testing problem.  Once the score function is fixed, any query point gets mapped to some value $v$. Typically, $v$ is compared to a threshold $T$: $v$ membership as IN if $v\ge T$, and OUT otherwise. This performs well, because the distributions $P_{\textrm{IN}}$ of IN scores and $P_{\textrm{OUT}}$ of OUT scores will differ. In practice, however, the attacker does not know the IN and OUT distributions, and so needs some side information to find a good threshold. All membership-inference attacks give the attacker some side information for this goal. In this work we consider a few different types of side information the attacker can have, corresponding to different ways of setting the threshold:

\begin{itemize}
    \item \emph{Batch Strategy.} The adversary has access to a dataset containing both update points (IN) and test points (OUT), but does not know which are IN.  Given these points, the adversary can compute scores and thereby see samples from the distributions $P_{\textrm{IN}}$ and $P_{\textrm{OUT}}$ and can then find an optimal threshold for distinguishing these points.  For example, if the attacker has $k$ update points and $k$ test points for large $k$, then the quantiles of these $2k$ scores will give a good threshold.  In our work we choose the median to maximize accuracy and the $10^{\textit{th}}$ percentile to maximize precision.  The assumption that the attacker has access to many update points is strong, but is a useful thought experiment, since the attacker can try to approximate these samples on their own using other forms of side information, which is what our next approach does.
    
    \item \emph{Transfer Strategy (``Shadow Models'').} The adversary has access to some set of test points.  Using these test points, the attacker trains \emph{shadow models}~\cite{shokri2017membership, salem2020updates, leino2020stolen} $\hat{f}_0$ and then updates these models using a random half of the test points.  Provided that the test points are drawn from the same (or similar) distribution and the attacker can train the models in the same (or similar) method as the real models, this should allow us to approximate $P_{\textrm{IN}}$ and $P_{\textrm{OUT}}$. Since the attacker knows which points were included in the update set, the attacker now has a batch of update points and test points, which means the attacker can use the batch strategy to obtain a good threshold for these points and then transfer that threshold to the models they are trying to attack. This only uses one shadow model, as we identify a threshold on a scalar value, rather than train a classifier, as many shadow model attacks do~\cite{salem2020updates, chen2020machine}.
    
    \item \emph{Rank Strategy.} The adversary again has access to some set of test points.  The attacker can use these to generate samples from $P_{\textrm{OUT}}$, and given these samples, can use the $q$-quantile of the distribution as a threshold.  This strategy will achieve a false-negative rate of $q$.  This strategy was employed by~\cite{leino2020stolen}.

\end{itemize}

\subsection{Analytical Justification}
\label{sec:algos-justification}

The preceding single-update attack algorithms rely on two key insights that our work contributes: (1) that taking advantage of model updates requires combining rich score information from individual-model MI attacks against the model and its update; and (2) that the two specific methods, \diff and \ratio, that we propose for combining two scores into a single MI attack in the model-update setting are both needed and justified.
This section provides analytical justification for these two insights.  Section~\ref{sec:evaluation} provides empirical evidence in support of these claims.

\paragraph{The need for rich score information.}
We justify our score combination approach by showing that model updates can only increase accuracy of membership-inference when the attacker can obtain rich (non-binary) score information.
We prove that, under reasonable assumptions, an adversary with only access to the binary scores does not improve when given access to the initial model.



This threat model where the attacker has 0/1 loss information is used for the baseline ``gap attack'' from prior work, which exploits trained models' generalization gaps~\cite{choo2020label}. In this attack, there are only two sets of points: correctly classified and incorrectly classified points. When the adversary has access to two models, there are four sets of points: $\lbrace f_0 \text{ Correct}, f_0 \text{ Incorrect} \rbrace\times \lbrace f_1 \text{ Correct}, f_1 \text{ Incorrect} \rbrace$. The adversary can make a decision for each of these sets. We define the frequency of each case with the following table:

\begin{table}[h!]
\small
\centering
\begin{tabular}{c|cc|cc}
~ & \multicolumn{2}{c|}{Update} & \multicolumn{2}{c}{Test} \\
~ & $f_1$ Correct & $f_1$ Incorrect & $f_1$ Correct & $f_1$ Incorrect \\

\hline
$f_0$ Correct & $p^{u}_{11}$ & $p^{u}_{10}$ & $p^{t}_{11}$ & $p^{t}_{10}$\\
$f_0$ Incorrect & $p^{u}_{01}$ & $p^{u}_{00}$ & $p^{t}_{01}$ & $p^{t}_{00}$
\end{tabular}

\label{tab:01theory}
\end{table}

Theorem~\ref{thm:01theory} proves that, under realistic assumptions, knowledge of $f_0$ does not improve the adversary's attack.  The assumptions are: (1) The initial model $f_0$ performs equally on update and test points: $p_{11}^u+p_{10}^u=p_{11}^t+p_{10}^t$.  This is realistic, as both sets of points are not in $f_0$'s training set.  (2) The updated model $f_1$ performs better on update points than on test, for both points correctly and incorrectly classified by $f_0$: $p_{11}^u>p_{11}^t,~ p_{01}^u>p_{01}^t$.  This is realistic, as models perform better on their training data than testing data.


\begin{theorem}
When the loss function is $0/1$-loss, and the assumptions above hold, then there is an attack that has access to only the loss on $f_1$ and has at least as high accuracy as any attack with access to the losses on both $f_0$ and $f_1$.
\label{thm:01theory}
\end{theorem}
\begin{proof}[Proof Sketch.]
The proof describes the optimal decisions for the adversary with and without updates. The optimal strategy without updates identifies correctly classified points as members, and incorrectly classified points as nonmembers. We show that, because the model improves performance on update points, the optimal strategy with updates is the same as without updates. The full proof is in the Appendix. 
\end{proof}

Thus, for an attack to successfully use model updates, it must exploit information beyond simply their generalization gaps.  This justifies our score combination approach in Section~\ref{sec:one_update}.

\paragraph{Justifying \diff and \ratio.}
We justify our choice of score combination functions by studying the example of computing the mean and the median of the training data.  These two examples will show both that our choices are well motivated, and also demonstrate that the right approach depends on the specific learning problem being solved, and thus there is likely not a single best method.  Section~\ref{sec:eval:q4} confirms these claims empirically.

The earliest work in membership-inference~\cite{Homer+08, SankararamanOJH09}, in the single-model/no-update setting, justified specific membership-inference attacks by exploiting a connection to hypothesis testing and use the Neyman-Pearson Lemma to devise optimal attacks.  Proving exact optimality typically requires making strong distributional assumptions, and being able to reason explicitly about the exact distribution of the outputs of the learning algorithm.  In our work, we mostly consider learning algorithms that are too complex for this sort of precise analysis (such as neural networks), so we settle for a more heuristic justification instead.  In particular, we will consider a learning algorithm that outputs the exact minimizer of the loss function to obtain the initial model $f_0$, then performs an update on a single point $x_{i}^{1}$ by performing a single gradient step from $f_0$ with fixed step size to obtain the updated model $f_1$.  This update strategy corresponds to what we call SGD-New, but with a single step of training.  We then analyze how the loss on the point $x_{i}^{1}$ changes as a result of the update and will see that the change is best reflected by either \diff or \ratio,  {\em depending on the loss function}.

We consider two loss functions: (1) the $\ell_2^2$ loss, which is denoted $\ell^m(x, f)=\|f-x\|_2^2$, and whose minimizer is the dataset's mean, and (2) the $\ell_2$ loss $\ell^g(x, f)=\|f-x\|_2$, whose minimizer is the dataset's geometric median.  For simplicity we also consider a single update point $x$.

\begin{algorithm}[t]
\DontPrintSemicolon
\KwData{Test Sample $x, y$, Models $f_0, f_1, \cdots f_k$, Privacy Score Function $\ell$, Threshold $T$ or Thresholds $\lbrace T_i\rbrace_{i=1}^k$, Strategy $\textsc{Attack}$ (\diff or \ratio)}

\SetKwFunction{BackFront}{BackFront}
\SetKwProg{Fn}{Function}{:}{}
\Fn{\BackFront{$x, y, f_0, f_1, \cdots, f_k, \ell, T, \textsc{Attack}$}}{
    \Return $\mathbbm{1}\left(\textsc{Attack}(x, y, f_0, f_k, \ell)<T\right)$
}

\SetKwFunction{Delta}{Delta}
\SetKwProg{Fn}{Function}{:}{}
\Fn{\Delta{$x, y, f_0, f_1, \cdots, f_k, \ell, \lbrace T_i\rbrace_{i=1}^k, \textsc{Attack}$}}{
  
    \For{$i\in 1 \dots k$}{
        $g_i =~\mathbbm{1}\left(\textsc{Attack}(x, y, f_{i-1}, f_i, \ell)<T_i\right)$ \\
        \lIf{$g_i =~$\emph{IN}}{
            \Return (IN, $i$)
        }
    }

\Return OUT
    
}

\caption{Multiple Update Attacks}
\label{alg:multi_update}
\end{algorithm}

For the $\ell_2^2$ loss, the update rule is
\[
f_1^m=f_0^m - \eta \nabla \ell^m(x, f_0^m)=f_0^m- 2\eta (f_0^m - x).
\]
Now we have,
\begin{align*}
    \ell^m(x, f_1^m)= & \left<(f_1^m-f_0^m) + (f_0^m-x), (f_1^m-f_0^m) + (f_0^m-x)\right> \\
    ={} & \|f_0^m-x|_2^2 + \|f_1^m-f_0^m\|_2^2 + 2\left<f_1^m-f_0^m, f_0^m-x\right>\\
    ={} & \ell^m(x, f_0^m) + (4\eta^2-4\eta)\|f_0^m-x\|_2^2 \\
    ={} & (1-2\eta)^2\ell^m(x,f_0^m).
\end{align*}
When recomputing the mean, the loss after update is a fixed ratio decrease from the loss before the update. 
Meanwhile, the probability that a randomly drawn test point will have the same loss ratio is 0. Then, with a single update point and a known learning rate, loss ratio is a \emph{perfect} membership test.


For the $\ell_2$ loss, the update rule is
\[
f_1^g
=f_0^g -\eta \nabla \ell^g(x, f_0^g)
=f_0^g- \eta \frac{f_0^g-x}{\|f_0^g-x\|_2}.
\]
Now we have,
\begin{align*}
    \ell^g(x, f_1^g)
    ={} & \sqrt{\| (f_1^g - f_0^g) + (f_0^g - x) \|_2^2} \\
    ={} & \sqrt{\|f_0^g-x\|_2^2 + \|f_1^g-f_0^g\|_2^2 + 2\left<f_1^g-f_0^g, f_0^g-x\right>}\\
    ={} & \sqrt{\ell^g(x, f_0^g)^2 + \frac{\eta^2\|f_0^g-x\|^2_2}{\|f_0^g-x\|^2_2}-2\eta\frac{\|f_0^g-x\|_2^2}{\|f_0^g-x\|_2}}\\
    ={} & \sqrt{\ell^g(x, f_0^g)^2+\eta^2-2\eta\ell^g(x, f_0^g)}\\
    ={} & \ell^g(x, f_0^g)-\eta,
\end{align*}
which shows that recomputing the geometric median results in a fixed constant decrease from the loss before the update, making it also a perfect membership test in this setting.

In the Appendix, we study mean estimation under updates, showing that model updates provably improve accuracy, and lower bounding the accuracy of \diff.

\subsection{Multiple Update Attacks}
\label{sec:mult_update}

Multiple updates can allow leakage in two ways: an adversary can learn both \emph{whether} a user is contained in a training set, but also \emph{when} they begin participating in that dataset. The former is the standard membership inference task, but when a user participated in a dataset may also be sensitive in cases where membership is sensitive, such as in medical datasets, where someone could learn when a patient contracted a disease.  We construct an attack for each goal: the Back-Front attack and the Delta attack, in Algorithm~\ref{alg:multi_update}.

\paragraph{Back-Front attack.} This attack ignores all information except for the first and last model update, and is designed for the generic accuracy case. This attack is the natural adaptation of the single model attacks to the generic accuracy case, as it treats the sequence of $k$ updates as a single, large, update algorithm. The attack uses either the score difference or ratio between the original model $f_0$ and the final model $f_k$ after $k$ updates.

\paragraph{Delta attack.} This attack is designed for the specific accuracy setting. In this attack, we identify samples to a specific update when they have a large loss difference (or ratio) on the two consecutive models $f_{i-1},f_{i}$ produced by the update.
Notice that we can also adapt this attack (or any specific accuracy attack) to the generic attack setting: a sample which is predicted to be in any update is predicted to be IN. As a result, we measure this attack's performance on both specific and generic accuracy. 

\paragraph{Threshold setting.} The three thresholding strategies we defined for the single-update setting,  Batch, Transfer, and Rank, can be adapted to configure the threshold for the multi-update attacks.
However, for simplicity, we focus on the Batch strategy.
With Batch, the adversary is given a dataset $D$ which contains each update set $\lbrace D_i\rbrace_{i=1}^k$ of size $n_{up}$, as well as a test set of size $k n_{up}$. For the Back-Front attack, we set the threshold  to the median value of the entire dataset $D$. For the Delta attack, we set each threshold $T_i$ so that $n_{up}$ points are classified into each update index.

%% file: evaluation-fewer-questions.tex
\begin{table}[t]
\centering
{
\begin{tabular}{l|p{0.20\linewidth}|p{0.35\linewidth}}
Method & Function & Description \\
\hline
\diff & single-update attack & compute difference of scores \\ \hline
\ratio & single-update attack & compute ratio of scores \\ \hline
Back-Front & multi-update attack & ignore intermediate models \\ \hline
Delta & multi-update attack & analyze adjacent models  \\ \hline
Transfer & setting threshold & transfer from a shadow model \\ \hline
Batch & setting threshold & calibrate with update/test points \\ \hline
Rank & setting threshold & calibrate with test points \\ \hline
\end{tabular}}
\caption{Proposed attacks and thresholding strategies.}
\label{tab:attacks}
\end{table}

\section{Evaluation}
\label{sec:evaluation}

We next evaluate our proposed algorithms for the single- and multi-update settings by answering seven key questions in the context of the datasets described in Section~\ref{sec:eval:datasets}:
\begin{itemize}
    \item[{\bf Q1:}] Does access to one model update give the attack an advantage over models with no access to updates? How does the update set size impact this advantage? (Section~\ref{sec:eval:q1})
    \item[{\bf Q2:}] Does attack advantage improve with the number of updates?  (Section~\ref{sec:eval:q2})
    \item[{\bf Q3:}] How does the training strategy---\sgdnew or \sgdfull---impact attack performance? (Section~\ref{sec:eval:q3})
    \item[{\bf Q4:}] How do the various attacks and thresholding choices impact attack performance? (Section~\ref{sec:eval:q4})
    \item[{\bf Q5:}] How does distribution shift impact attack performance? (Section~\ref{sec:eval:q5})
    \item[{\bf Q6:}] How would adoption of differential privacy impact attack performance? (Section~\ref{sec:eval:q6})
    \item[{\bf Q7:}] How do our attacks compare to those developed for unlearning in Chen et al.~\cite{chen2020machine}? (Section~\ref{sec:eval:q7})
\end{itemize}

Since we propose multiple attacks for two different settings -- single-update and multiple updates -- and each attack can be instantiated with multiple score functions and thresholding methods, the space of experimentation is quite large.  Moreover, performance of an attack can be measured in different ways, such as with accuracy, precision, and recall for the single-update setting, and specific accuracy and generic accuracy for the multiple-update setting.
To tame this large experimentation space, we answer different questions in the context of different algorithms, methods, settings, and metrics that are most relevant for the specific question.
Q1-Q3 choose the best-performing attack relevant to the considered setting -- single or multiple updates -- and focus on accuracy and specific accuracy, respectively.
Q4 compares key pairs of algorithms and mechanisms under multiple performance metrics.
And Q5-Q7 focus on the single-update setting and best-performing algorithms.

\subsection{Datasets}
\label{sec:eval:datasets}

In this section we describe the datasets used in our evaluation.

\paragraph{FMNIST.} FMNIST is a 10-class dataset consisting of 28x28 pixel grayscale images of different clothing items. On this dataset, the initial model is a logistic regression model, trained on an initial dataset of 1000 data points for 50 epochs at a learning rate of 0.01. On average, this achieves 82.5\% training accuracy and 79.5\% test accuracy. \sgdnew trains for 10 epochs at a learning rate of 0.001, and \sgdfull trains for 10 epochs at a learning rate of 0.01.

\paragraph{CIFAR-10.} CIFAR-10 is a 10-class dataset of 32x32 pixel RGB images of various animals and vehicles. This dataset is harder than FMNIST, and requires more complex models to achieve reasonable accuracy. We fine-tune a VGG-16 network which was pretrained on the ImageNet dataset. The initial model is trained with 12 epochs over a training set of 25000 points at a learning rate of $10^{-4}$. \sgdnew trains for 4 epochs at a learning rate of $10^{-5}$, and \sgdfull trains for 2 epochs at a learning rate of $10^{-4}$.

\paragraph{Purchase100.} Purchase100 is a 100-class purchase history dataset. The task is to classify a shopper into one of 100 clusters. Here, our initial model is a single layer neural network trained on an initial dataset of 25000 samples for 95 epochs at a learning rate of 0.01. \sgdnew trains for 5 epochs at a learning rate of 0.01, and \sgdfull trains for 10 epochs at a learning rate of 0.1.

\paragraph{IMDb.} IMDb is a text dataset of movie reviews, where the task is to classify a movie review as either positive or negative. Here, we fine tune the BERT base model (uncased) which was pretrained on a large collection of English data. The initial model is trained with 4 epochs over a training set of 25000 points at a learning rate of $10^{-5}$. \sgdnew trains for 6 epochs at a learning rate of $10^{-6}$, and \sgdfull trains for 3 epochs at a learning rate of $10^{-5}$.

\subsection{Advantage from a Single Update}
\label{sec:eval:q1}

\begin{figure*}[t]
    \centering
\begin{subfigure}{0.4\textwidth}
  \includegraphics[width=\linewidth]{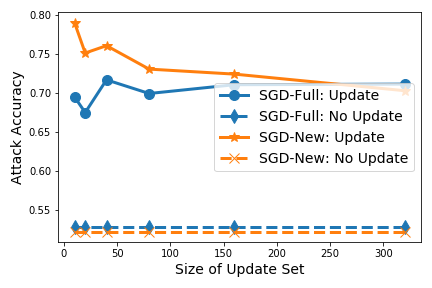}
  \caption{FMNIST}
  \label{fig:pou:fmnist_lbfgs}
\end{subfigure}
\begin{subfigure}{0.4\textwidth}
  \includegraphics[width=\linewidth]{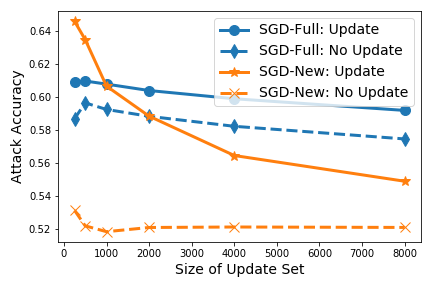}
  \caption{CIFAR-10}
  \label{fig:pou:cifar_sgd}
\end{subfigure}
\begin{subfigure}{0.4\textwidth}
  \includegraphics[width=\linewidth]{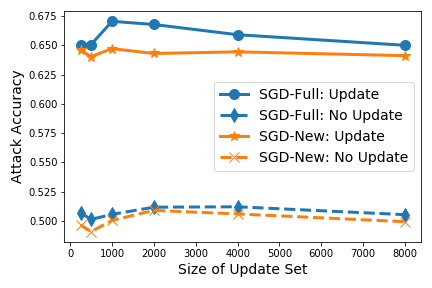}
  \caption{Purchase100}
  \label{fig:pou:purchase_sgd}
\end{subfigure}
\begin{subfigure}{0.4\textwidth}
  \includegraphics[width=\linewidth]{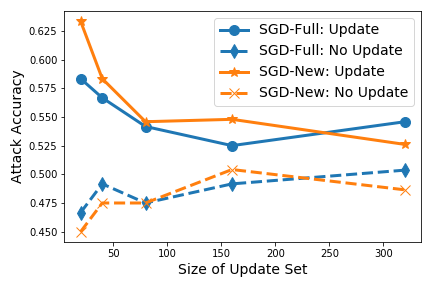}
  \caption{IMDb}
  \label{fig:pou:imdb_sgd}
\end{subfigure}
\caption{(Q1, Q3) Accuracy of attacks with access to one update (solid lines) vs. without access to updates (dashed lines).  Each line corresponds to the best attack for each setting (see text).
Even one update gives the adversary significant power to identify training set members.}
\label{fig:single_update}
\end{figure*}

{\bf Q1: Does access to one model update give the attack an advantage over models with no access to updates? How does the update set size impact this advantage?}

To evaluate MI advantage from a single update, we compare with three baseline attacks that use only the updated model $f_1$. The first baseline, called {\em Loss}, uses the approach of Yeom et al.~\cite{yeom2018privacy}, which compares the loss on a point to the average training loss. The second baseline, called {\em Gap}, uses the gap attack of \cite{choo2020label}, which classifies correctly classified points as training and incorrectly classified points as test. For all datasets except IMDb, we also evaluate use the LiRA attack~\cite{carlini2021membership} as a baseline, which trains shadow models to compute sample-specific baseline loss values to compare to.

Figure~\ref{fig:single_update} shows the accuracy of the best of our single-update attacks, compared to the best baseline without access to updates.  We show these accuracies for update sizes varying from 1\% to 32\% of the original training set (except IMDb, where we use a fixed 10-320 points for acceptable running time).  The best attack differs in each setting.  For updates, fixing the Batch thresholding strategy, we choose the best of $\{\text{\diff}, \text{\ratio}\} \times \{\text{loss score},\text{LiRA score}\}$.  For no updates, we choose the best of $\{\text{Gap}, \text{Loss}, \text{LiRA}\}$ baselines.
For both update and no-update, we show accuracy for both training strategies \sgdfull and \sgdnew.

For all datasets, update sizes, and training strategies, attacks with model updates outperform the no-update attacks.
As expected, the gap between updates and no updates decreases as the update set gets larger. On FMNIST, for example, at 10 update points, the Batch attack achieves 79\% accuracy, while the Batch attack achieves 70\% accuracy at 320 update points. The gap between the update and no update attacks decreases from 27\% to 18\%.

\paragraph{Q1 answer:} Overall, our results show that updates give the adversary significant advantage to identify training set members.

\subsection{Advantage from Multiple Updates}
\label{sec:eval:q2}

\begin{figure*}[t]
    \centering
\begin{subfigure}{0.4\textwidth}
  \includegraphics[width=\linewidth]{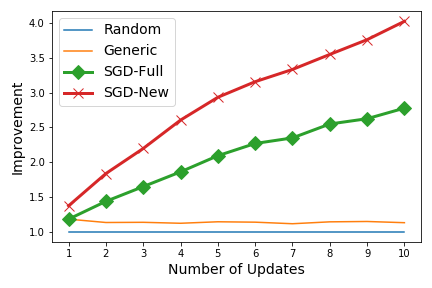}
  \caption{FMNIST}
  \label{fig:mult:spec_fmnist}
\end{subfigure}
\begin{subfigure}{0.4\textwidth}
  \includegraphics[width=\linewidth]{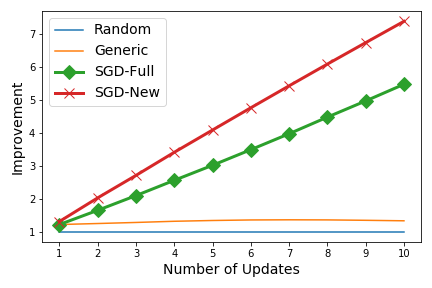}
  \caption{CIFAR-10}
  \label{fig:mult:spec_cifar}
\end{subfigure}
\begin{subfigure}{0.4\textwidth}
  \includegraphics[width=\linewidth]{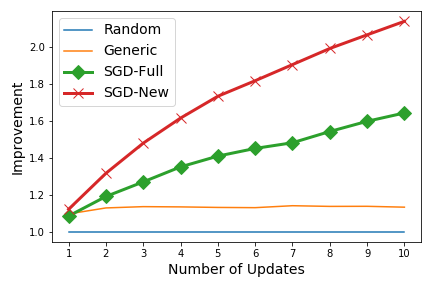}
  \caption{Purchase100}
  \label{fig:mult:spec_purchase}
\end{subfigure}
\begin{subfigure}{0.4\textwidth}
  \includegraphics[width=\linewidth]{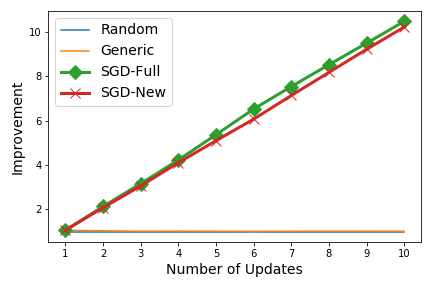}
  \caption{IMDb}
  \label{fig:mult:spec_imdb}
\end{subfigure}
\caption{(Q2, Q3) Improvement of specific accuracy over random guessing (y axis), as a function of the number of updates (x axis). For $k$ updates, Random guessing achieves $1/2k$ accuracy. We present the best attack strategy for each dataset: for FMNIST and Purchase100, this is loss difference, but loss ratio for CIFAR-10. As the number of updates increases, the attack performs significantly better than random guessing, although the absolute accuracy decays.}
\label{fig:mult_specific}
\end{figure*}

{\bf Q2: Does advantage improve with the number of updates?}

To evaluate the threat of MI with the number of updates, we run our multi-update attacks on a sequence of 1 to 10 updates, and observe how each attack's performance changes. To isolate the role of multiple updates, we fix $\nup$ to 250 for CIFAR-10, 100 for Purchase100, and 10 for FMNIST and IMDb (1\%, .4\%, 1\%, and .04\% of the initial dataset, respectively). We measure both specific and generic accuracy, which are the metrics relevant for the multi-update attack, however we focus in this section on specific accuracy and evaluate generic accuracy in a subsequent section. However, we remind the reader that a generic accuracy attack can be used to construct a specific accuracy attack, which we will use as a baseline.

We compare against two baselines. The first, called {\em Random}, is random guessing. It randomly selects IN or OUT with probability $1/2$ and selects the update index uniformly at random. This results in a success probability of $1/2k$ when there are $k$ updates. We also define a stronger baseline, called {\em Generic}, which runs a generic accuracy attack to answer IN or OUT, and selects the update uniformly at random. We ignore the details of generic accuracy attacks for this section. If the generic accuracy of the attack is $p$, then it obtains specific accuracy $p/k$. When generic accuracy performs better than random chance ($p>1/2$), Generic outperforms Random.

Figure~\ref{fig:mult_specific} shows the multiplicative improvement of specific accuracy over random guessing as a function of the number of updates. As before, we show the best attack for each case, fixing the thresholding strategy to Batch.  Our best attacks significantly outperform the baselines for specific accuracy, with the gap increasing with the number of updates. For example, for \sgdnew on CIFAR-10, at 2 updates, our best attack (in this case, \ratio\ with cross-entropy loss as the score) outperforms the baseline by a factor of $2.03\times$, while at 8 updates, it outperforms the baseline by $6.10\times$. Still, for both the baselines and our attacks, the absolute value of specific accuracy decays with the number of updates.  This is natural, as increasing the number of updates makes it more difficult to guess which of the many updates a point was used in---this is why the Random and Generic baselines' success probabilities -- $1/2k$ and $p/k$, respectively -- decay with $k$.  For our attack, the absolute value of specific accuracy for CIFAR-10 50.8\% at 2 updates and 38.1\% at 8 updates.  Yet, our attacks fare much better compared to our baselines, whose specific accuracy for CIFAR-10 is 25\% at 2 updates and 6.25\% at 8 updates.

\paragraph{Q2 answer:} Thus, as the number of updates increases, the attack performs significantly better than appropriate baselines, although the absolute accuracy fundamentally decreases for all cases.

\subsection{Impact of Training Strategy}
\label{sec:eval:q3}

{\bf Q3: How does the training strategy---\sgdnew or \sgdfull---impact attack performance?}

The success of the MI attack depends on the training strategy, which is a choice the learner makes that the adversary cannot influence.
If one training strategy were consistently less vulnerable to attack than the other, then the learner could choose the former, as a heuristic defense.
We find that such a heuristic exists, but it depends on the update size and dataset, making it difficult for the learner to configure without experimentation with our attacks.

We revisit Figures~\ref{fig:single_update} and~\ref{fig:mult_specific}, which already illustrate the effect of \sgdnew and \sgdfull on attack performance with one and multiple updates, respectively.
For a single update, Figure~\ref{fig:single_update} shows the attack is more effective with \sgdnew when the update set is small.  However, the attack's performance on \sgdnew degrades faster than on \sgdfull as the update set size increases.  For example, on CIFAR, the gap between 250 and 8000 update points is 10\% for \sgdnew but only 2\% for \sgdfull.
This faster drop in performance with \sgdnew on larger update sets makes \sgdfull, in fact, more vulnerable to attack for large update sets on CIFAR-10 compared to \sgdnew.  The reason for this inversion is that \sgdfull would naturally use a higher learning rate for the update set compared to \sgdnew. Interestingly, the inversion already happens on Purchase100 at the smallest update set size.

For multiple updates, we observe a similar and even more consistent effect.
Figure~\ref{fig:mult_specific}, which uses a very small update set of 1\% of the original training set, shows the attack as most effective with \sgdnew.  This is true even more consistently across datasets than with a single update.
In the multiple update setting, \sgdnew only uses each update point once, so an update point will observe a larger loss decrease in the update in which it appears compared to the loss in updates in which it does not appear.  \sgdfull, by contrast, should observe a loss decrease in each update.
We include in Appendix~\ref{app:more_exp} a figure (Figure~\ref{fig:specific100}) showing that for a larger update set size of 10\% of the original training set, attacks on \sgdfull are more effective relative to attacks on \sgdnew.

\paragraph{Q3 answer:} Thus, as a rule of thumb, when update sets are small, training on the whole dataset is less vulnerable to MI attack; but when update sets are larger, training on only the update set is less vulnerable.  The point at which the inversion happens depends on the dataset, so a learner should tune this heuristic, for example, by running our attack.  Section~\ref{sec:eval:q5} evaluates differential privacy as a far more principled approach for defense against MI attacks, yet still one that  will likely require some tuning or auditing in practice, for which our attacks could still prove useful.

\subsection{Impact of Attack Strategy}
\label{sec:eval:q4}

{\bf Q4: How do the various attacks and thresholding choices impact attack performance?}

In the preceding questions, we picked the best performing attack algorithm and thresholding strategy from the suite we are proposing.
But are all these attacks/thresholds really needed, or do some of them outperform others consistently?
The answer is, indeed, that different algorithms and thresholds perform better in different settings, so they are all relevant in an attacker's toolkit. For example, the LiRA score function is more difficult to compute than loss, so should be used when it can be reliably estimated. Recall that the adversary's knowledge determines their choice of threshold.
We show here a few main comparisons.

\begin{table*}
    \centering \footnotesize
    \tabcolsep=0.15cm
    \begin{tabular}{c||cc|cc|cc||cc|cc|cc}
        \hline
        \multirow{3}{*}{Dataset} & \multicolumn{6}{c||}{\sgdnew} & \multicolumn{6}{c}{\sgdfull} \\
        ~ & \multicolumn{2}{c|}{No Update} & \multicolumn{2}{c|}{$\textsc{ScoreDiff}$ (loss)} & \multicolumn{2}{c||}{$\textsc{ScoreRatio}$ (loss)} & \multicolumn{2}{c|}{No Update} & \multicolumn{2}{c|}{$\textsc{ScoreDiff}$ (loss)} & \multicolumn{2}{c}{$\textsc{ScoreRatio}$ (loss)}\\
        & Gap & Loss & Batch & Transfer & Batch & Transfer & Gap & Loss & Batch & Transfer & Batch & Transfer\\
        \hline
        FMNIST & 0.51 & 0.53 & 0.68 & 0.67 & 0.71 & 0.72 & 0.52 & 0.51 & 0.63 & 0.55 & 0.57 & 0.58 \\
        CIFAR-10 & 0.50 & 0.53 & 0.61 & 0.62 & 0.65 & 0.68 & 0.50 & 0.59 & 0.58 & 0.59 & 0.61 & 0.63 \\
        Purchase100 & 0.51 & 0.50 & 0.56 & 0.56 & 0.57 & 0.58 & 0.51 & 0.51 & 0.54 & 0.54 & 0.55 & 0.54 \\
        IMDb & 0.50 & 0.45 & 0.63 & 0.60 & 0.63 & 0.60 & 0.50 & 0.47 & 0.58 & 0.55 & 0.58 & 0.55 
    \end{tabular}
    \caption{(Q4) Accuracy for single update models with smallest update size.
    }
    \label{tab:single_update_accuracy}
\end{table*}

\paragraph{\diff\ vs. \ratio\ in terms of accuracy.} Focusing first on the single update setting, Table~\ref{tab:single_update_accuracy} shows attack accuracy with an update size of 1\% of the original training set.
We compare \diff\ and \ratio\ when using the standard cross-entropy loss and either the Batch or Transfer thresholding strategy (Batch and Transfer are the only ones relevant for accuracy).
We show results for both \sgdnew and \sgdfull training strategies.
The \ratio\ strategy typically outperforms \diff.  However, there are exceptions: at large update set sizes, loss difference is comparable or sometimes outperforms ratio. The two threshold selection strategies, Batch and Transfer, have comparable performance in the cases we show.
We also ran the same evaluation with LiRA as the score in \diff\ and \ratio.
LiRA consistently performs better than traditional loss, except on CIFAR-10, where there is not enough data to train sufficient shadow models for it to perform well. We find LiRA tends to be strongest with \ratio, but \diff\ performs comparably.

For the multiple update setting, there is similarly no single best choice for the attacker.  For example, loss ratio outperforms loss difference on CIFAR-10, while loss difference outperforms loss ratio on FMNIST and Purchase100, and there is little difference on IMDb.

For a specific use case, an adversary might run an experiment on a shadow updated model they themselves generate to determine whether to use \diff\ or \ratio, and which privacy score to use.   Thresholding strategies are determined by what the adversary has access to.

\begin{table}
    \small
    \centering
    \begin{tabular}{cc|cccc}
        \hline
        Dataset & $n_{\mathrm{up}}$ & Loss & Batch & Transfer & Rank \\
        \hline
        FMNIST & 10 & 0.53/0.11 & 0.78/0.16 & 0.86/0.24 & 0.82/0.28 \\
        CIFAR-10 & 250 & 0.54/0.10 & 0.85/0.17 & 0.82/0.22 & 0.69/0.46 \\
        Purchase100 & 250 & 0.53/0.11 & 0.60/0.12 & 0.63/0.07 & 0.60/0.14 \\
        IMDb & 10 & 0.58/0.12 & 0.58/0.12 & 0.56/0.32 & 0.52/0.25 \\
        \hline
        FMNIST & 80 & 0.52/0.10 & 0.74/0.15 & 0.78/0.13 & 0.69/0.25 \\
        CIFAR-10 & 2000 & 0.55/0.11 & 0.83/0.17 & 0.79/0.25 & 0.69/0.44 \\
        Purchase100 & 2000 & 0.49/0.10 & 0.63/0.13 & 0.65/0.10 & 0.61/0.16 \\
        IMDb & 80 & 0.46/0.09 & 0.67/0.13 & 0.70/0.12 & 0.59/0.25 \\
    \end{tabular}
    \caption{(Q4) Precision/recall results for single update models on \sgdfull. Rank sets the threshold to the top 10\% of the test data. We report values for two update set sizes. It is well known that no update loss obtains poor precision~\cite{leino2020stolen, carlini2021membership}.}
    \label{tab:single_update_precision}
\end{table}

\paragraph{\diff\ vs. \ratio\ in terms of precision/recall.}
The broad takeaways for precision are similar to those for accuracy. \sgdnew outperforms \sgdfull, as we can see from comparing Table~\ref{tab:single_update_precision} and a table we included in the Appendix, Table~\ref{tab:precision_only}. At $\nup=0.01n_0$, Batch achieves 90\% precision on FMNIST for \sgdnew, while the best attack achieves 86\% precision for \sgdfull. This difference is more pronounced for other datasets. Precision is larger when fewer update points are used: Transfer achieves 86\% precision with $\nup=0.01n_0$ on FMNIST, but only 78\% precision with $\nup=0.08n_0$. For precision, we have three strategies for selecting the threshold: Batch, Transfer, and also Rank.  While Rank often achieves a higher recall than the other strategies, Batch and Transfer typically achieve a higher precision.
Thus, once again, the choice of algorithm and its configurations requires experimentation.

\begin{figure*}[t]
    \centering
\begin{subfigure}{0.4\textwidth}
  \includegraphics[width=\linewidth]{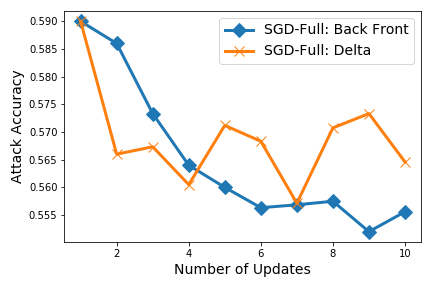}
  \caption{FMNIST}
  \label{fig:mult:fmnist_generic}
\end{subfigure}
\begin{subfigure}{0.4\textwidth}
  \includegraphics[width=\linewidth]{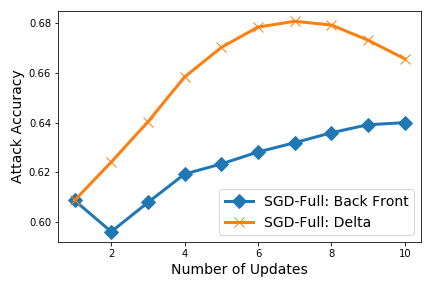}
  \caption{CIFAR-10}
  \label{fig:mult:cifar_generic}
\end{subfigure}
\begin{subfigure}{0.4\textwidth}
  \includegraphics[width=\linewidth]{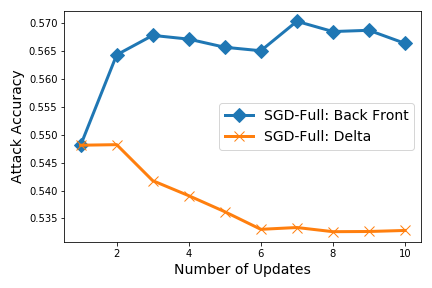}
  \caption{Purchase100}
  \label{fig:mult:purchase_generic}
\end{subfigure}
\begin{subfigure}{0.4\textwidth}
  \includegraphics[width=\linewidth]{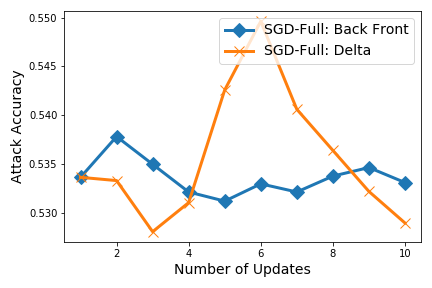}
  \caption{IMDb}
  \label{fig:mult:imdb_generic}
\end{subfigure}
\caption{(Q4) Generic accuracy for attacks with multiple updates. We select the best strategy for each case: for FMNIST, this is always loss difference (although ratio performs similarly), and is loss ratio for CIFAR-10 and Purchase100.}
\label{fig:mult_generic}
\end{figure*}

\paragraph{Delta vs. Back-Front in terms of generic accuracy.}  For the multiple update setting, we proposed two generic accuracy strategies for the attacker: Delta, which compares each adjacent pair of models; and Back-Front, which only compares the first and last models.
Figure~\ref{fig:mult_generic} compares these two strategies for \sgdfull in terms of the generic accuracy.
As before, no strategy strictly dominates. The Delta attack performs the best for CIFAR-10. On Purchase100, the Back-Front attack is best. On FMNIST, the Back-Front and Delta attack perform comparably.
As before, an adversary could decide between one strategy versus another based on experiments with a shadow model that they train themselves.

\paragraph{Q4 answer:}
These results show that the variety of algorithms and configurations that we proposed is truly needed, because no single algorithm and configuration will work best in all situations.

\subsection{Impact of Distribution Shift}
\label{sec:eval:q5}

{\bf Q5: How does distribution shift impact attack performance?}

An important characteristic of the model retraining setting we tackle in this paper is that models are updated to keep up with a shifting distribution.  How distribution shift impacts MI attack performance has been evaluated very little in prior literature, hence its investigation, in the context of our proposed attacks, is a significant contribution of our work.  The contribution consists of two components: (1) a new methodology that we developed for evaluating impact of distribution shift on MI attacks and (2) the evaluation of our proposed attack algorithms with this methodology.

\paragraph{Methodology.}
We focus on {\em subpopulation shift} and use the BREEDS framework~\cite{santurkar2020breeds}.
The BREEDS framework generates subpopulation shift by generating a hierarchy of classes and shifting between classes which are close in the hierarchy. For example, a ``dog'' class trained on images of dalmatians may struggle to recognize poodles as dogs. We adapt BREEDS to CIFAR-10 with an ``animal vs. vehicle'' binary task, and vary the animals and vehicles to simulate distribution shift.  ``Animals'' are $\lbrace$ bird, cat, deer, dog, frog, horse$\rbrace$; ``vehicles'' are $\lbrace$ airplane, automobile, ship, truck $\rbrace$. 
For the ``animal'' class, we consider a source class $s_a$ and a target class $t_a$ (both class 0 in the binary task). Likewise, the ``vehicle'' class has a source class $s_v$ and target class $t_v$ (both class 1). We write the distribution of a given class $c$ as $\mathcal{D}^c$. We consider balanced classes, so that the source distribution is $\mathcal{D}_S = \tfrac{1}{2}(\mathcal{D}^{s_a}+\mathcal{D}^{s_v})$ and the target distribution $\mathcal{D}_T = \tfrac{1}{2}(\mathcal{D}^{t_a}+\mathcal{D}^{t_v})$. The original training distribution is $\mathcal{D}_S$, and the update distribution is $\mathcal{D}_1 = (1-\alpha)\mathcal{D}_S + \alpha\mathcal{D}_T$, where the shift ratio $\alpha$ controls the strength of the distribution shift.

We run our single update attacks from Section~\ref{sec:one_update} on the preceding distribution shift methodology. Our goal is to isolate the role of distribution shift on attack performance.
Keeping everything else constant, we vary the parameter $\alpha$ and consider two settings of $(s_a \rightarrow t_a), (s_v \rightarrow t_v)$. The first, which we call {\em Hard}, is (Airplane $\rightarrow$ Automobile), (Cat $\rightarrow$ Bird). The second, which we call {\em Easy}, is (Automobile $\rightarrow$ Truck), (Cat $\rightarrow$ Dog). Hard is a distribution shift where the original model will not perform well on the new data, due to the dissimilarity between original and update classes.  Easy is a distribution shift where the original model will perform well. 

Importantly, our methodology does {\em not} measure the ability of an adversary to distinguish the old and new distributions.  This is because our MI game formulation from Section~\ref{sec:mi_updates} samples test points and update points identically, from the same distribution. Instead, we are measuring the ability to distinguish {\em shifted training points from shifted testing points}.  Intuitively, a large distribution shift requires the model to fit to the specific update points to accommodate the new distribution, thereby making the update points vulnerable to MI.
It is worth noting that the only prior work investigating distribution shift's impact on privacy, namely Zanella-B{\'e}guelin et al.~\cite{zanella2020analyzing}, shows that some outputs are more likely after a distribution shift, but their experiments cannot isolate whether this is a privacy violation or just the model adapting to the new distribution.  We thus believe that our methodology can constitute a better platform for future measurements of MI under distribution shift.

\begin{table*}
    \centering \footnotesize
    \tabcolsep=0.15cm
    \begin{tabular}{cc|cccccc|cccccc}
        \hline
        \multirow{3}{*}{Metric} & \multirow{3}{*}{Shift} & \multicolumn{6}{c|}{\sgdnew} & \multicolumn{6}{c}{\sgdfull} \\
        ~ & ~ & \multicolumn{2}{c}{$\alpha=0.2$} & \multicolumn{2}{c}{$\alpha=0.6$} & \multicolumn{2}{c|}{$\alpha=1.0$} & \multicolumn{2}{c}{$\alpha=0.2$} & \multicolumn{2}{c}{$\alpha=0.6$} & \multicolumn{2}{c}{$\alpha=1.0$}\\
        ~ & ~ & \textsc{Diff} & \textsc{Ratio} & \textsc{Diff} & \textsc{Ratio} & \textsc{Diff} & \textsc{Ratio} & \textsc{Diff} & \textsc{Ratio} & \textsc{Diff} & \textsc{Ratio} & \textsc{Diff} & \textsc{Ratio}\\
        \hline
        \multirow{2}{*}{Accuracy} & Hard & 0.551 & 0.559 & 0.602 & 0.634 & 0.600 & 0.687 & 0.574 & 0.597 & 0.578 & 0.585 & 0.595 & 0.666 \\
        & Easy & 0.685 & 0.685 & 0.574 & 0.581 & 0.560 & 0.562 & 0.597 & 0.606 & 0.570 & 0.570 & 0.598 & 0.598 \\
        \hline
        \multirow{2}{*}{Prec./Recall} & Hard & .73/.15 & .90/.18 & .94/.19 & .96/.19 & .98/.20 & .94/.19 & .63/.13 & .76/.15 & .69/.14 & .95/.19 & .80/.16 & .94/.19 \\
        & Easy & .64/.13 & .74/.15 & .58/.12 & .90/.18 & .71/.14 & .94/.19 & .54/.11 & .64/.13 & .47/.09 & .70/.14 & .62/.12 & .86/.17 \\
    \end{tabular}
    \caption{(Q5) Accuracy and precision/recall of attacks after subpopulation shift.
    $\textsc{Ratio}$, $\textsc{Diff}$ stand for \ratio, \diff, respectively, with loss score.
    $\textsc{Ratio}$ performs best, a Hard distribution shift results in better attacks as $\alpha$ increases, and \sgdnew is typically more vulnerable than \sgdfull.}
    \label{tab:shift_accuracy}
\end{table*}

\paragraph{Evaluation.}
The second component of our distribution-shift contribution is the evaluation of our proposed algorithms using the preceding methodology.
We focus on the single update attacks, fix the Batch thresholding strategy, and $\nup=50$, $n_0=5000$.
We evaluate accuracy and precision/recall for $\alpha \in \lbrace 0.2, 0.6, 1.0 \rbrace$, shown in Table~\ref{tab:shift_accuracy}. We show results for \sgdfull and \sgdnew, using both \diff and \ratio with the loss score.

{\em Accuracy:} \ratio is the most effective strategy in all cases.
For the Hard distribution shift, the faster the change in distribution (larger $\alpha$), the more effective MI is. As $\alpha$ increases from 0.2 to 1.0, accuracy on \sgdfull increases by 0.07, and accuracy increases by 0.13 on \sgdnew.
Similar to how \sgdnew is more influenced by $\nup$ in Section~\ref{sec:eval:q3}, \sgdnew is also more heavily influenced by $\alpha$.
Interestingly, accuracy decreases as $\alpha$ grows for the Easy shift. This shows that a drastic shift requires significant changes to be made to the model, leading the model to overfit to the specific update points and make them more vulnerable.

{\em Precision/Recall:} The lessons for precision/recall are similar to those for accuracy, but precision tends to be much higher than accuracy, often reaching >90\% (and as high as 98\%), especially at more drastic shifts. We still find that \ratio is generally more effective, and Hard shifts and large $\alpha$ result in higher precision.
The key difference with accuracy we note is that even the Easy shift results in large precision at high $\alpha$; this is likely because, even when the distributions are similar, there are some samples which still require the model to change significantly to learn them. 

\paragraph{Q5 answer:} A drastic distribution shift can result in significantly higher membership inference risk than a gradual shift.

\subsection{Differential Privacy}
\label{sec:eval:q6}

{\bf Q6: How does differential privacy impact our attacks?}

A rigorous strategy for preventing our attacks is by training with differential privacy. Indeed, training with $(\varepsilon, \delta)$-differential privacy, imposes some upper bound on the accuracy of \emph{any} membership-inference attacks and also on the precision at a fixed level of recall.  We evaluate the effectiveness of this strategy by training with DP-SGD~\cite{song2013stochastic, bassily2014private, abadi2016deep}, the standard algorithm for training differentially private models. This algorithm modifies standard SGD by clipping gradient norms and adding noise. We use the implementation provided by Tensorflow Privacy~\cite{tfprivacy}. We fix $\delta=10^{-4}$, and vary the noise multiplier to vary between fixed values of $\varepsilon$, computed with the accounting provided by the repository. We focus on the Fashion-MNIST dataset, with an update size of $100$, a clipping norm of 0.5, and we fix all other parameters to be the same as in previous sections. We focus on the single update setting for simplicity.

\paragraph{Protection wanes with $\epsilon$.}  
We present the results of this experiment in Figure~\ref{fig:dp}. As expected, as $\epsilon$ increases, its protection from our attacks decreases. For example, at $\epsilon=0.26$, our best attack reaches a precision of only $0.59$, but at $\epsilon=1.1$, it reaches a precision of $0.63$. The precision levels off as $\epsilon$ increases. The gap between the attacks with and without access to updates is also largest at moderate $\epsilon$ values.  At $\epsilon < 1$, there is little difference.  The no-update attacks also catch up at extremely large $\epsilon$, where the noise addition is minor, and gradient-clipping is the major difference between DP-SGD and the standard implementation of SGD.  

\begin{figure}[t]
\centering
\begin{subfigure}{0.4\textwidth}
  \includegraphics[width=\linewidth]{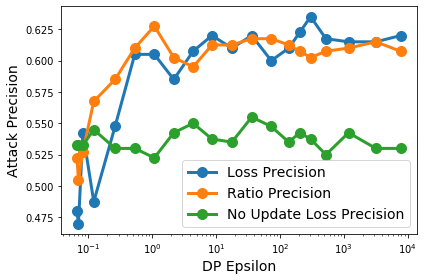}
  \caption{Attacking DP training}
  \label{fig:dp}
\end{subfigure}
\begin{subfigure}{0.4\textwidth}
  \includegraphics[width=\linewidth]{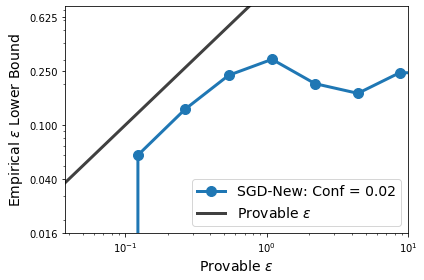}
  \caption{Auditing $\varepsilon$}
  \label{fig:epslb}
\end{subfigure}
\caption{(Q6) (a) Effectiveness of differential privacy at defending against single update attacks, as $\varepsilon$ varies for \sgdnew.
Accuracy and precision for both \sgdnew and \sgdfull observe similar behavior, where no update attacks are similar at small $\varepsilon$, but fall behind at larger $\varepsilon$.
(b) Using our attacks to audit differentially private \sgdfull and \sgdnew. All lower bounds for \sgdfull are 0. For multiple values of $\varepsilon$, \sgdnew lower bounds are a 2.0-3.6x factor from upper bounds.}
\end{figure}

\paragraph{Auditing differentially private deployments with our attacks.}  
Following \cite{jagielski2020auditing, NasrSTPC21}, we can use our results to provide empirical lower bounds on the privacy of each update algorithm, as a means of understanding how worst-case upper bounds on $\epsilon$ correspond to practical privacy against state-of-the-art attacks. We can convert the bound used by either \cite{jagielski2020auditing} or \cite{NasrSTPC21} to a bound on precision, giving that precision $p$ should be bounded by $e^{\varepsilon} / (1+e^{\varepsilon})$, or, identically, that $\varepsilon$ is lower bounded by $p/(1-p)$. Since we cannot measure $p$ directly, we follow~\cite{jagielski2020auditing} and compute conservative estimates via Clopper-Pearson confidence intervals.  We have 400 trials for each $\varepsilon$, as we train 20 models with 20 points in each trial. We report in Figure~\ref{fig:epslb} these computed $\varepsilon$ values for both \sgdnew and \sgdfull. We enforce a confidence of 98\% for each reported value.

We highlight two key takeaways.  First, as in non-private training, \sgdnew\ empirically offers less privacy protection than \sgdfull\ for those points in the update.  In fact, our attack does not refute the possibility that \sgdfull\ satisfies differential privacy with $\epsilon = 0$, although we stress that this is not robust evidence that privacy is not a concern when retraining with \sgdfull.  The second takeaway is that the provable upper bounds on the privacy of \sgdnew\ are nearly tight for moderate values of $\epsilon$.  With provable $\varepsilon$ in .12-1.09, our lower bounds are within a 2.0-3.6x factor of the theoretical upper bound.  A gap smaller than 3.6x is perhaps remarkable, since state-of-the-art attacks on standalone models trained with DP-SGD have gaps of 5-10x~\cite{jagielski2020auditing, NasrSTPC21}, suggesting that model updates represent an especially risky scenario for private model training.

\paragraph{Q6 answer:} Our results show that differential privacy is an effective protection at lower $\varepsilon$, and also that our attacks can be an effective method of empirically auditing a differential privacy deployment.

\subsection{Comparison with Chen et al. [9]}
\label{sec:eval:q7}

{\bf Q7: How do our attacks compare to those developed for unlearning in Chen et al.~\cite{chen2020machine}?}

We observe that machine unlearning can be viewed as the ``reverse operation'' of our model update setting. Then we can adapt the MI attacks designed for machine unlearning in Chen et al.~\cite{chen2020machine}, and compare them to ours. In their strategy, the adversary trains shadow models to learn an ``attack model''. This attack model takes as input some combination of the probability vectors returned by the two models, and outputs a prediction for whether the point was unlearned. They experiment with different instantiations of the attack, and we reproduce their SortedDiff attack in our setting, which takes the difference between sorted probability vectors before and after deletion. They note SortedDiff is their best attack on well-generalized models, as our models are. We run their attack with up to 30 shadow models, with each of the attack model architectures tested, and report the best attack from these. We focus on our loss score attacks with Transfer thresholds, as this fits the threat model they consider (LiRA requires training different types of shadow models, so we avoid this comparison for simplicity). Our attacks here will therefore only use a single shadow model (to set the threshold), while their attacks will be allowed up to 30.

We compare \diff\ and \ratio\ with the Chen et al.~\cite{chen2020machine} attack in Figure~\ref{fig:chenetal}, for both \sgdfull and \sgdnew, on the FMNIST dataset. We observe that both \diff\ and \ratio\ always outperform their attack, although the gap can be somewhat small, depending on the update size. We are able to do this with a single shadow model, because it is easier to identify a good threshold on a single feature, than learn a good function on 10 features\footnote{We also note that their attacks are much stronger on overfit models, and the attack they performed on MNIST models that do not exhibit a large amount of overfitting only achieved an AUC of .51.}. We corroborate this on Purchase100.

To further demonstrate the strength of the test statistics we compute, we run an experiment allowing the Chen et al. attack model to access our \diff\ and \ratio\ features, in addition to the features they use (this results in 12 features, two of ours, and 10 SortedDiff features from their paper). We allow this improved Chen et al. attack 30 shadow models, and have it learn a logistic regression attack model on these 12 features. We inspect the weights the model learns, as a way to measure how useful the features are, and find that the average weight assigned to our features is, on average, 7.5x higher than the weight assigned to one of the SortedDiff features! This speaks to the value of carefully constructing a useful test statistic, rather than attempting to learn one from a high dimensional space. Shadow models are more useful when used to improve a simple test statistic, as our results with LiRA show.

\paragraph{Q7 answer:}
\diff\ and \ratio\ are more efficient (fewer shadow models) and more effective (higher accuracy) than~\cite{chen2020machine}.

\begin{figure}[t]
\centering
\begin{subfigure}{0.4\textwidth}
  \includegraphics[width=\linewidth]{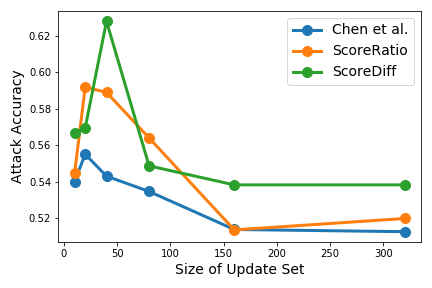}
  \caption{\sgdfull}
  \label{fig:chenetal_full}
\end{subfigure}
\begin{subfigure}{0.4\textwidth}
  \includegraphics[width=\linewidth]{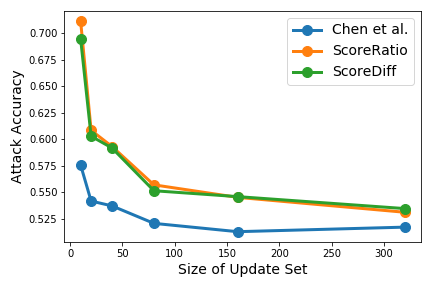}
  \caption{\sgdnew}
  \label{fig:chenetal_new}
\end{subfigure}

\caption{(Q7) Performance of Chen et al.~\cite{chen2020machine} relative to our attacks. We consistently outperform their attack, despite only using a single shadow model to set our attack threshold.}
\label{fig:chenetal}
\end{figure}

%% file: disc.tex
\section{Conclusions}
We have presented and evaluated MI attacks which leverage model updates. Our attacks apply to a variety of settings, including when models are repeatedly updated and when the distribution shifts over time. Our strategies are theoretically justified and empirically achieve both high accuracy and high precision. Empirically, we find the role of the update set size, the training algorithm, and any distribution shift to be key factors impacting our attacks' performance. As a general rule, the smaller the update set size, the more effective our attacks are. This holds true for attacks on single updates as well as multiple updates, and for both \sgdnew and \sgdfull. It is well known that MI attacks are more successful when training sets are smaller~\cite{SankararamanOJH09, dwork2015robust}, so our results confirm this in the model updates setting. A drastic distribution shift also improves the performance of an attack, as learning the new distribution requires fitting heavily to the update points. Zanella-B{\'e}guelin et al.~\cite{zanella2020analyzing} also notice that distribution shift results in memorization in the generative language model setting, but our results are the first to identify this as privacy leakage, rather than the model adapting to a new distribution. Finally, the specific training setup used by the learner can impact the accuracy of MI. Models updated repeatedly can be used to improve MI, and can also leak the time that a data point appeared in the training set. Learners training with \sgdnew\ are more vulnerable than \sgdfull\ at small update set sizes and when distributions shift significantly, but this trend tends to reverse as the update set gets larger and the distribution is more stable.

%% file: appendices.tex
\section{Proof of Theorem~\ref{thm:01theory}}
In this section we prove Theorem~\ref{thm:01theory}, showing that 0/1 losses aren't helpful for model updates.  We restate the theorem and its assumptions.
\begin{itemize}
    \item The initial model $f_0$ performs equally well on update and test points: $p_{11}^u+p_{10}^u=p_{11}^t+p_{10}^t$.  This is a reasonable assumption as both sets are not in $f_0$'s training set.
    \item The updated model $f_1$ performs better on update points than on test, for both points correctly and incorrectly classified by $f_0$: $p_{11}^u>p_{11}^t,~ p_{01}^u>p_{01}^t$.  This is realistic, as models perform better on their training data than testing data.
\end{itemize}

\begin{theorem}[Theorem \ref{thm:01theory}, restated]
If the above assumptions hold, then no attack which uses the 0/1 loss of both $f_0$ and $f_1$ can achieve better accuracy than an attack which uses the 0/1 loss of $f_1$.
\end{theorem}
\begin{proof}
Given only knowledge of $f_1$, the adversary may only make a decision based on whether or not $f_1$ correctly classified the point. Then a uniform decision must be made for all correctly classified points, representing a $p_{11}^u+p_{01}^u$ fraction of update points and a $p_{11}^t+p_{01}^t$ fraction of test points. Because there is a balance between update and test points, there will be more updates points in this set when $p_{11}^u+p_{01}^u>p_{11}^t+p_{01}^t$, so the optimal attack should classify IN when this inequality holds and OUT otherwise. The fraction of overall points which are correctly classified by this rule will be $\tfrac{1}{2}(\max(p_{11}^u+p_{01}^u, p_{11}^t+p_{01}^t))$. Similarly, points which are incorrectly classified by $f_1$ should be classified as IN if $p_{10}^u+p_{00}^u > p_{10}^t+p_{00}^t$ and OUT otherwise. Then the optimal attack achieves accuracy \[\tfrac{1}{2}\left[\max(p_{11}^u+p_{01}^u, p_{11}^t+p_{01}^t) + \max(p_{10}^u+p_{00}^u, p_{10}^t+p_{00}^t)\right].\]

Now, Assumption 2 states that the model improves on update data more than on test data, so $\max(p_{11}^u+p_{01}^u, p_{11}^t+p_{01}^t)=p_{11}^u+p_{01}^u$; the optimal attack should identify correctly classified points as belonging to the update set. Similarly, $\max(p_{10}^u+p_{00}^u, p_{10}^t+p_{00}^t)=p_{10}^t+p_{00}^t$, so the optimal attack should classify incorrectly classified points as belonging to the test set. This optimal attack strategy reaches an accuracy of $\tfrac{1}{2}(p_{11}^u+p_{01}^u+p_{10}^t+p_{00}^t)$.

\paragraph{Using Updates.}
Now, an attack with access to both $f_0$ and $f_1$ may make decisions based on both classifications. Following the same argument as before, we see that a $p_{11}^u$ fraction of update points are correctly classified by both models, while a $p_{11}^t$ fraction of test points are, so the optimal attack should classify these points as update if $p_{11}^u>p_{11}^t$ and test otherwise. Applying the same logic to all four possible classifications, we see that the optimal attack achieves accuracy
\begin{align*}
    \frac{1}{2}\left[\max(p_{11}^u, p_{11}^t)+\max(p_{01}^u, p_{01}^t)\right] + \frac{1}{2}\left[ \max(p_{10}^u, p_{10}^t)+\max(p_{00}^u, p_{00}^t)\right].
\end{align*}

Again, Assumption 2 gives $\max(p_{11}^u, p_{11}^t)=p_{11}^u$ and $\max(p_{01}^u, p_{01}^t)=p_{01}^u$. Assumption 1 states that training and test accuracy on $f_0$ are identical, so combining Assumption 1 and Assumption 2 gives $p_{10}^u<p_{10}^t$, so that $\max(p_{10}^u, p_{10}^t)=p_{10}^t$. Finally, Assumption 1 implies that $p_{01}^u+p_{00}^u=p_{01}^t+p_{00}^t$. Combining this with Assumption 2 gives us that $p_{00}^u<p_{00}^t$, so $\max(p_{00}^u, p_{00}^t)=p_{00}^t$. Then the accuracy of the optimal attack when given both $f_0$ and $f_1$ is 
\[
\frac{1}{2}\left[p_{11}^u+p_{01}^u+p_{10}^t+p_{00}^t\right],
\]
equivalent to the attack which did not use updates. The attack in both cases is identical - identify points correctly classified by $f_1$ as IN, and those incorrectly classified as OUT.
\end{proof}

\section{Additional Theoretical Analysis}

\subsection{Mean Estimation}
\label{ssec:meanest}

In this section we give a more detailed treatment of Example~\ref{ex:mean-est} from the Introduction.  Namely, we analyze the effect of updates in a very simple case: a single update of a rudimentary ``model'' that estimates the mean over a multi-dimensional dataset.

We consider the task of estimating the mean of samples drawn from a $d$-dimensional spherical Gaussian distribution $\mathcal{D}=\mathcal{N}(\mu,\mathbb{I}_{d\times d})$. We consider a learner $\Atrain=\Aup$ which simply outputs the sample mean of its training data, and which produces two models---the first, $\hat{\mu}_0$, is computed on a dataset $D_0$ of $n_0$ samples, and the second, $\hat{\mu}_1$, is computed with an additional $n_1$ samples. The total dataset $D=[D_0; D_1]$ contains $n=n_0+n_1$ samples. We consider an adversary who seeks to identify whether a given $v$ was contained in $D_1$.
In this setting, we can upper bound the performance of \emph{any} attack when the adversary has no access to model updates (that is, it only has access to $\hat{\mu}_1$). Next, we show that, when given access to model updates (both $\hat{\mu}_0$ and $\hat{\mu}_1$), the adversary can outperform this upper bound.

\begin{theorem} 
Consider membership-inference using mean estimation in $\mathbb{R}^d$ with an initial dataset of size $n_0$ and a single model update with a set of size $n_1$.  If
\begin{equation}
\frac{1}{2} + \sqrt{\frac{2d}{n-1}} < \Phi\left(\sqrt{\frac{d}{80(n_1-1)}}\right)
\label{eq:thm1}
\end{equation}
then there is an attacker with access to both $\hat\mu_0$ and $\hat\mu_1$ that outperforms every adversary with access to only $\hat\mu_1$.  Here $\Phi(z) = \Pr[N(0,1) \leq z]$ is the Gaussian CDF.
\label{thm:meanest}
\end{theorem}
The condition \eqref{eq:thm1} is always satisfied when $n \gg d \gg n_1$, in which case the left-hand side is close to $1/2$ and the right-hand side is close to $1$.  We now prove this statment, by breaking it into two lemmas. Lemma~\ref{lemma:noup} upper bounds all adversaries with access to only $\hat\mu_1$. Lemma~\ref{lemma:dot} analyzes a specific attack using $\hat\mu_0$ and $\hat\mu_1$ together. Combining these lemmas proves Theorem~\ref{thm:meanest}.

\begin{lemma}
For the task of membership inference of a sample $v$ on mean estimation, with probability $>1-\exp(-d)$ over the selection of $v$, all adversaries $\mathcal{A}$ have success rate bounded above by
\[\textsc{Acc}(\mathcal{A})=\frac{1}{2}+\frac{1}{2}\left(\sqrt{\frac{5d}{n-1}} + \frac{\sqrt{d}}{n-1}\right).\]
\label{lemma:noup}
\end{lemma}
\begin{proof}
Notice that $\hat{\mu}_0$ is distributed as $\mathcal{N}(\mu, \tfrac{\sigma^2}{n_0}\mathbb{I})$ and $\hat{\mu}_1$ as $\mathcal{N}(\mu, \tfrac{\sigma^2}{n_0+n_1}\mathbb{I})$. An adversary seeking to distinguish between the cases where a sample $v$ is contained in $D_1$ or not must distinguish between two distributions: the distribution over means where $v$ is not contained in $D_1$, $\mathcal{D}_{\OUTrm}=\mathcal{N}(\mu, \tfrac{\sigma^2}{n}\mathbb{I})$, and the distribution over means where $v$ is contained in $D_1$, $\mathcal{D}_{\INrm}=\tfrac{1}{n}v+\mathcal{N}(\tfrac{n-1}{n}\mu, \tfrac{\sigma^2(n-1)}{n^2}\mathbb{I})$. We can upper bound the success of this adversary by a function of the total variation (TV) distance between the two distributions: $\tfrac{1}{2}+\tfrac{1}{2}TV(\mathcal{D}_{\OUTrm}, \mathcal{D}_{\INrm})$.

Recall that the TV distance between $\mathcal{N}(x_0, \Sigma_0)$ and $\mathcal{N}(x_1, \Sigma_1)$ can be upper bounded~\cite{devroye2018total} by
\begin{align*} 
TV(\mathcal{N}(x_0, \Sigma_0), \mathcal{N}(x_1, \Sigma_1)\le & \left\|\Sigma_0^{-1/2}(x_0-x_1)\right\|_2 + \\ & \left\|\mathbb{I} - \Sigma_0^{-1/2} \Sigma_1 \Sigma_0^{-1/2}\right\|_F.
\end{align*}

We use this to bound $TV(\mathcal{D}_{\OUTrm}, \mathcal{D}_{\INrm})$ as follows:
\begin{align*} 
TV(\mathcal{D}_{\OUTrm}, \mathcal{D}_{\INrm}) \le
& \left\|\left(\frac{n}{\sigma\sqrt{n-1}}\mathbb{I}\right)\left(\frac{1}{n}v-\frac{1}{n}\mu\right)\right\|_2 + \\
  & \left\|\mathbb{I} - \frac{n}{\sigma\sqrt{n-1}}\mathbb{I}\left(\frac{\sigma^2}{n}\mathbb{I}\right)\frac{n}{\sigma\sqrt{n-1}}\mathbb{I}\right\|_F \\
 ={}& \frac{\|v - \mu\|_2}{\sigma\sqrt{n-1}} + \left\|\frac{1}{n-1}\mathbb{I}\right\|_F \\
 ={}& \frac{\|v - \mu\|_2}{\sigma\sqrt{n-1}} + \frac{\sqrt{d}}{n-1}.
\end{align*}

Because $v-\mu \sim \mathcal{N}(0, \sigma^2 \mathbb{I})$, we have $\|v-\mu\|_2<\sigma\sqrt{5d}$ \cite{10.2307/2674095} except with probability at most $\exp(-d)$, which gives us
\[
TV(\mathcal{D}_{\OUTrm}, \mathcal{D}_{\INrm}) \le \sqrt{\frac{5d}{n-1}} + \frac{\sqrt{d}}{n-1}.
\]
This completes the proof.
\end{proof}

\begin{lemma}
For the task of membership inference of a sample $v$ on mean estimation with model updates, there exists an adversary $\mathcal{A}$ with success rate
\[
\textsc{Acc}(\mathcal{A})\ge \Phi\left(\sqrt{\frac{d}{80(n_1-1)}}\right).
\]
This holds with probability $>1-2\exp(-d/16)$ over the choice of $v$, and when $n_1>1$.
\label{lemma:dot}
\end{lemma}

\begin{proof}
We consider an adversary with access to model updates, receiving two quantities. The first is the mean of $D_0$, $\hat{\mu}_0\sim\mathcal{N}(\mu, \tfrac{\sigma^2}{n_0}\mathbb{I})$. The next is the overall mean $\hat{\mu}_1$. When $v$ is not contained in $D_1$, $\hat{\mu}_1$ is distributed as $\tfrac{n_0}{n}\hat{\mu}_0 + \tfrac{n_1}{n}\mathcal{N}(\mu, \tfrac{\sigma^2}{n_1}\mathbb{I})$. When $v$ is found in $D_1$, $\hat{\mu}_1$ is distributed as $\tfrac{n_0}{n}\hat{\mu}_0 +  \tfrac{1}{n}v + \tfrac{n_1-1}{n}\mathcal{N}(\mu, \tfrac{\sigma^2}{n_1-1}\mathbb{I})$.

With both of these quantities, the adversary computes the mean of only $D_1$: $\hat{\mu}_{\Delta}=\tfrac{n}{n_1}\hat{\mu}_1 - \tfrac{n_0}{n_1}\hat{\mu}_0$. The task of determining whether $v$ is contained in $D_1$ can now be written as the task of distinguishing between the distribution of $\hat{\mu}_{\Delta}$ when $v$ is not included, $\mathcal{D}_{\Delta, \OUTrm}$, and the distribution when it is included, $\mathcal{D}_{\Delta, \INrm}$, both written below:
\begin{equation*}
\mathcal{D}_{\Delta, \OUTrm}=\mathcal{N}\left(\mu, \frac{\sigma^2}{n_1}\mathbb{I}\right) \textrm{ and }
\mathcal{D}_{\Delta, \INrm}=\frac{1}{n_1}v + \mathcal{N}\left(\frac{n_1-1}{n_1}\mu, \frac{\sigma^2(n_1-1)}{n_1^2}\mathbb{I}\right).
\end{equation*}

Now, the adversary computes $s(\hat{\mu}_{\Delta}, v) = (\hat{\mu}_{\Delta}-\mu)\cdot (v-\mu)$. 
In the OUT case, $s(\hat{\mu}_{\Delta}, v)$ is distributed as $\mathcal{N}(0, \tfrac{\|v-\mu\|_2^2\sigma^2}{n_1})$. In the IN case, it is distributed as $\mathcal{N}(\tfrac{1}{n_1}\|v-\mu\|_2^2, \tfrac{\|v-\mu\|_2^2\sigma^2(n_1-1)}{n_1^2})$.

The adversary guesses OUT if $s(\hat{\mu}_{\Delta}, v)$ is below $T=\tfrac{1}{2n_1}\|v-\mu\|_2^2$ and IN otherwise. If $\mathcal{D}_{\Delta, \INrm}$ and $\mathcal{D}_{\Delta, \OUTrm}$ had equal variance, this would be the optimal Neyman-Pearson distinguisher~\cite{neyman1933ix}; because the variances are similar, the test will still be effective. For convenience, we write $c=\|v-\mu\|_2^2\sigma^2$.

The probability the adversary succeeds when $v$ is OUT is
\[
Acc_{\OUTrm} = \Pr\left[\mathcal{N}\left(0, \tfrac{c}{n_1}\right)\le T\right],
\]
and the probability of success when $v$ is IN is 
\[
Acc_{\INrm} = \Pr\left[T \le \mathcal{N}\left(2T, \tfrac{c(n_1-1)}{n_1^2}\right)\right]=\Pr\left[\mathcal{N}\left(0, \tfrac{c(n_1-1)}{n_1^2}\right)\le T\right].
\]
The adversary achieves accuracy $Acc = \tfrac{1}{2}(Acc_{\OUTrm} + Acc_{\INrm})$.

We have $Acc_{\INrm}\le Acc_{\OUTrm}$ because of its higher variance. Then we can use proceed by computing $Acc_{\INrm}$ as a lower bound for $Acc$. To compute this probability, we notice that, due to the distribution of $v$, we have $\|v-\mu\|_2<\sigma\sqrt{5d}$ with probability $>1-\exp(-d)$, and $\|v-\mu\|_2>\sigma\sqrt{d/2}$ with probability $>1-\exp(-d/16)$ \cite{10.2307/2674095}. We can therefore lower bound $T> d\sigma^2 / 4n_1$ and upper bound $c<5d\sigma^4$, in all giving
\begin{align*}
\textsc{Acc}(\mathcal{A}) >{} & \Pr\left[\mathcal{N}\left(0, \frac{5d\sigma^4(n_1-1)}{n_1^2}\right)\le \frac{d\sigma^2}{4n_1}\right] \\
={} & \Phi\left(\sqrt{\frac{d}{80(n_1-1)}}\right)
\end{align*}
This completes the proof.
\end{proof}

\subsection{\textsc{ScoreDiff} Achieves High Accuracy}
We now show that \diff with the loss score on mean estimation achieves high accuracy when $d>>n_1$.

\begin{theorem}
Suppose $f_0$ is the mean of $D_0$, and $f_1$ is produced by taking a single gradient step from $f_0$ with a learning rate of $\eta$ on the $\ell_m$ loss. Then there is some constant $C$ and threshold $T$ such that, if $d>Cn_1$, running \diff with a threshold of $T$ reaches a membership inference accuracy of >90\% for both SGD-Full and \sgdnew.
\label{thm:gradient}
\end{theorem}
In the following, we write the mean of $D_1$ as $\mu_{\uprm}$. When excluding a sample $x$, we write the mean of $D_1/x$ as $\mu_{\restrm}$.
We begin by proving Lemma~\ref{lemma:gradient}.
\begin{lemma}
For the task of mean estimation with an update $D_1$, when $f_0$ is the mean of the original dataset $D_0$, a gradient step with learning rate $\eta$ using SGD-Full is equal to a gradient step using \sgdnew\ with a learning rate of $\eta'=\tfrac{n_1\eta}{n_0+n_1}$.
\label{lemma:gradient}
\end{lemma}
\begin{proof}
In \sgdnew, the gradient step is performed on $D_1$ is
\[
-2\eta\sum_{x_i\in D_1}\frac{f_0-x_i}{\nup}=2\eta(\mu_{\uprm}-f_0).
\]
In SGD-Full, the gradient step on $D_0\cup D_1$ is
\[
-2\eta\sum_{x_i\in D_0\cup D_1}\frac{f_0-x_i}{n_0+\nup}=\frac{2\nup\eta}{n_0+\nup}(\mu_{\uprm}-f_0),
\]
as the gradient on $D_0$ adds to 0, because $f_0$ is the minimizer of $\ell_m$ on $D_0$. We see that these gradient steps are rescalings of each other, as we wanted to show.
\end{proof}

Having proven Lemma~\ref{lemma:gradient}, we can now prove Theorem~\ref{thm:gradient}, by fixing an $\eta$ and analyzing the loss difference with \sgdnew.

\begin{proof}
We write the loss for a fixed $x$ and $\eta$, and will later consider the cases where $x$ is IN $D_1$ and where $x$ is a test point, OUT of $D_1$.

\begin{align*}
&\ell^m(f_1, x) \\ = & \left<f_1 - x, f_1 - x\right> \\
= & \left<(f_1 - f_0) + (f_0 - x), (f_1 - f_0) + (f_0 - x)\right> \\
= & \|f_1-f_0\|_2^2 + \|f_0-x\|_2^2 + 2\left<f_1 - f_0, f_0 - x\right>\\
= & \ell^m(f_0, x_i) + \|2\eta(\mu_{\uprm} - f_0)\|_2^2 - 4\eta\left<\mu_{\uprm}-f_0, x-f_0\right>.
\end{align*}

Now, the norm $\|2\eta(\mu_{\uprm} - f_0)\|_2^2$ can be computed from $f_0$ and $f_1$, so we consider only the distribution of the rightmost term $d=- 4\eta\left<\mu_{\uprm}-f_0, x-f_0\right>$, showing that this is smaller for update points than for test points. In the OUT case, we write $d$ as $d_{\OUTrm}$, which is distributed as
\begin{align*}
&d_{\OUTrm} \\ 
= & -4\eta\left(\left<\mu_{\uprm}-\mu, x-\mu\right>+\|f_0-\mu\|_2^2-\left<x-\mu, f_0-\mu\right>-\right.\\
& ~~~~~~~~~~\left.\left<\mu_{\uprm}-\mu, f_0-\mu\right>\right)\\
= & -4\eta\left(\tfrac{\sigma^2}{n_0}X_1+\sigma^2\sqrt{\tfrac{n_0+n_1}{n_0n_1}}\left(X_2-X_3\right)+\tfrac{\sigma^2}{\sqrt{n_0n_1}}\left(X_4-X_5\right)\right),
\end{align*}
where each of the $X_i$ are independent samples from a Chi square distribution $\chi^2_d$. 
The mean of $d_{\OUTrm}$ is $\tfrac{-4\eta}{n_0}d\sigma^2$, and its variance is $32d\eta^2\sigma^4\left(\tfrac{n_1+2n_0^2+2n_1n_0+2n_0}{n_0^2n_1}\right)=O\left(d\eta^2\sigma^4\left(\tfrac{1}{n_0}+\tfrac{1}{n_1}\right)\right)$.

In the IN case, we write $d$ as $d_{\INrm}$, distributed as
\begin{align*}
&d_{\INrm} \\= & -4\eta\left(\left<\mu_{\uprm}-\mu, x-\mu\right>+\|f_0-\mu\|_2^2-\left<x-\mu, f_0-\mu\right>-\right.\\
& ~~~~~~~~~~\left.\left<\mu_{\uprm}-\mu, f_0-\mu\right>\right)\\
= & -4\eta\left(\tfrac{1}{n_1}\|x-\mu\|_2^2+\tfrac{\sigma^2}{n_0}X_1+\tfrac{\sigma^2(n_1+1)}{n_1\sqrt{n_0}}(X_2-X_3)\right. \\
& ~~~~~~~~~~~\left.\tfrac{\sigma^2(n_1-1)}{n_1\sqrt{n_1}}(X_4-X_5)+\tfrac{\sigma^2(n_1-1)}{n_1\sqrt{n_0n_1}}(X_6-X_7)\right),
\end{align*}
where each of the $X_i$ are independent samples from a Chi square distribution $\chi^2_d$.
The mean of $d_{\INrm}$ is $\tfrac{-4\eta}{n_0}d\sigma^2-\tfrac{4\eta}{n_1}\|x-\mu\|_2^2$, and its variance is $O\left(d\eta^2\sigma^4\left(\tfrac{1}{n_0}+\tfrac{1}{n_1}\right)\right)$.

Now, the difference in means for $d_{\OUTrm}$ and $d_{\INrm}$ is $4\eta\|x-\mu\|_2^2/n_1$, which is distributed as $4\eta\chi^2_d/n_1$, which we can bound below by $d\eta\sigma^2/n_1$ with probability $>1-\exp(-d/16)$ \cite{10.2307/2674095} over the randomness of selecting $x$. Then we can consider a threshold of $\tfrac{1}{2}(\mathbb{E}[d_{\OUTrm}]+\mathbb{E}[d_{\INrm}])\ge \mathbb{E}[d_{\OUTrm}]+0.5d\eta\sigma^2/n_1$. The success at distinguishing between the IN and OUT case is then bounded below by $Acc>\Pr\left[d_{\OUTrm}<\mathbb{E}[d_{\OUTrm}]+0.5d\eta\sigma^2/n_1\right]$. To measure this probability, notice that the variance of $d_{\OUTrm}$ is $Var[d_{\OUTrm}]=O\left(\tfrac{d\eta^2\sigma^4}{n_1}\right)$, if we assume $n_1<n_0$. Then there is some $C_{\OUTrm}$ for which the variance $Var[d_{\OUTrm}]<\tfrac{C_{\OUTrm}d\eta^2\sigma^4}{n_1}$. Then we can use Chebyshev's Inequality\footnote{For a random variable $Y$, $\Pr[|Y-\mathbb{E}[Y]|\ge k\sqrt{Var(Y)}]\le\tfrac{1}{k^2}$.} to bound our success probability from below as
\begin{align*}
   Acc=&\Pr\left[d_{\OUTrm}<\mathbb{E}[d_{\OUTrm}]+\sqrt{Var[d_{\OUTrm}]}\sqrt{\tfrac{d}{8C_{\OUTrm}n_1}}\right]\\
   \ge & 1-\frac{8C_{\OUTrm}n_1}{d}.
\end{align*}
The analysis for the IN case is identical, although requires a different constant $C_{\INrm}$, due to the different variance term. Then, to guarantee $>90\%$ accuracy as in the theorem statement, we can take $K=80\max(C_{\INrm}, C_{\OUTrm})$.
\end{proof}

\section{Extended Experiments}
\label{app:more_exp}

\begin{table*}
    \centering
    \begin{tabular}{c|cccc|cccc}
        \hline
        \multirow{2}{*}{Dataset} & \multicolumn{4}{c|}{$\nup=0.01n_0$} & \multicolumn{4}{c}{$\nup=.08n_0$}\\
        & Loss & Batch & Transfer & Test & Loss & Batch & Transfer & Test \\
        \hline
        FMNIST & .53/.11 & .90/.18 & .80/.48 & .82/.38 & .51/.10 & .71/.14 & .83/.09 & .66/.20 \\
        CIFAR-10 & .47/.09 & .92/.18 & .90/.06 & .83/.46 & .49/.10 & .83/.17 & .84/.14 & .77/.36\\
        Purchase100 & .54/.11 & .70/.14 & .77/.09 & .68/.18 & .49/.10 & .54/.11 & .54/.22 & .54/.12 \\
        IMDb & .50/.10 & .83/.17 & .70/.34 & .75/.22 & .52/.10 & .70/.14 & .74/.14 & .69/.16   
    \end{tabular}
    \caption{Precision/recall results for single update models on \sgdnew. Loss, Batch, Transfer, and Rank are defined in Table~\ref{tab:single_update_accuracy}.  We report for two $\nup$ values.}
    \label{tab:precision_only}
\end{table*}
\begin{algorithm}
\KwData{Test Sample $x$, Estimated Mean $\hat{\mu}$, True Mean $\mu$, True Variance $\sigma^2$, Dimensions $d$, Dataset Size $n$}

\SetKwFunction{NeyPea}{NeymanPearson}
\SetKwProg{Fn}{Function}{:}{}
\Fn{\NeyPea{$x, \hat{\mu}, \mu, \sigma, d, n$}}{
    $p_{\INrm} = \left(\frac{n-1}{2\pi\sigma^2}\right)^{d/2}\exp\left(-\tfrac{n-1}{2\sigma^2}\|\hat{\mu}-\tfrac{n-1}{n}\mu-\tfrac{1}{n}x\|_2^2\right)$\\
    $p_{\OUTrm} = \left(\frac{n}{2\pi\sigma^2}\right)^{d/2}\exp\left(-\tfrac{n}{2\sigma^2}\|\hat{\mu}-\mu\|_2^2\right)$ \\
    \Return $\mathbbm{1}(p_{\INrm} > p_{\OUTrm})$
}
\caption{Neyman-Pearson Optimal~\cite{neyman1933ix} Distinguisher for Mean Estimation}
\label{alg:neypea}
\end{algorithm}
\paragraph{Mean Estimation}

\begin{figure}
  \centering
  \includegraphics[width=.4\linewidth]{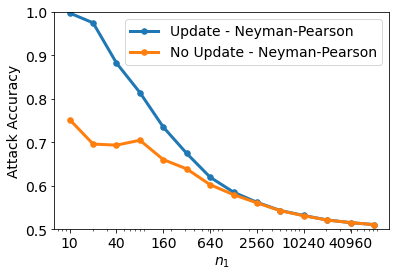}
\caption{Demonstrating the power of updates. Here, we use $n_0=200$, $d=250$, $\mu=0$, $\sigma = 0.1$, and vary $n_1$ along the x axis, reporting performance averaged over 60 trials.}
\label{fig:neypea}
\end{figure}

To validate the improved performance of model updates for mean estimation, we test attacks on the mean estimation setup described earlier in the Appendix. For these experiments, we set $\mu=\mathbf{0}$, $n_0=200$, $d=250$, $\sigma=0.1$, and we vary $n_1$. We experiment an attack which leverages updates, and one which doesn't.

In the no update case, we run the Neyman-Pearson optimal attack~\cite{neyman1933ix}, as described in Algorithm~\ref{alg:neypea}. This attack, when provided a $\hat{\mu}_1$, computes the PDF values $p_{\OUTrm}=\mathcal{D}_{\OUTrm}(\hat{\mu}_1)$ and $p_{\INrm}=\mathcal{D}_{\INrm}(\hat{\mu}_1)$, and returns IN ($g_0=1$) if $p_{\INrm}>p_{\OUTrm}$ and OUT ($g_0=0$) otherwise. For Algorithm~\ref{alg:neypea}, we set $\hat{\mu}=\hat{\mu}_1$ and $n=n_0+\nup$.

In the update case, we run the corresponding Neyman-Pearson optimal attack, which is equivalent to the no update Neyman-Pearson optimal attack, replacing $\hat{\mu}_1$ with $\hat{\mu}_{\Delta}$. When provided $\hat{\mu}_1$ and $\hat{\mu}_0$, the attack computes the mean of $D_1$ as $\hat{\mu}_{\Delta}=\tfrac{n}{n_1}\hat{\mu}_1 - \tfrac{n_0}{n_1}\hat{\mu}_0$, and then computes the PDF values $p_{\OUTrm}=\mathcal{D}_{\Delta, \OUTrm}(\hat{\mu}_1)$ and $p_{\INrm}=\mathcal{D}_{\Delta, \INrm}(\hat{\mu}_1)$, where the distributions $\mathcal{D}_{\Delta, \OUTrm}$ and $\mathcal{D}_{\Delta, \INrm}$ are defined in the proof of Lemma~\ref{lemma:dot}. Then the attack returns IN ($g_0=1$) if $p_{\INrm}>p_{\OUTrm}$ and OUT ($g_0=0$) otherwise. For Algorithm~\ref{alg:neypea}, we set $\hat{\mu}=\hat{\mu}_{\Delta}$ and $n=\nup$.

In Figure~\ref{fig:neypea}, we empirically evaluate these attacks. The immediate takeaway is that using model update never results in worse performance than not using updates. However, notice that this message holds in a variety of scenarios, even those which are not covered by our theoretical analysis. Because $d=250$ and $n_0=200$, when $n_1<50$, we see that even when $d>n$, the model update attack outperforms the no update attack, a setting our analysis performs poorly in. Also, even when $n_1$ is much larger than $n_0$, we see the attack which uses updates does not perform poorly relative to the attack which does not. The improved performance from using model updates is robust to a wide range of parameter settings.

\begin{figure}
    \centering
    \begin{subfigure}{0.4\textwidth}
    \includegraphics[width=\linewidth]{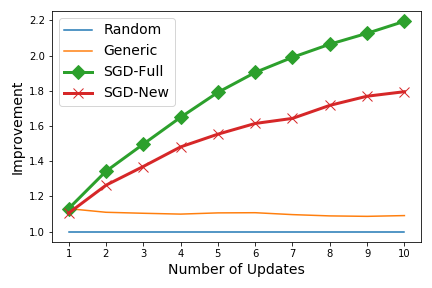}
    \caption{Specific Accuracy}
    \label{fig:specific100}
    \end{subfigure}
    \begin{subfigure}{0.4\textwidth}
    \includegraphics[width=\linewidth]{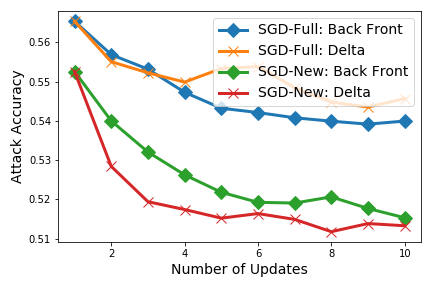}
    \caption{Generic Accuracy}
    \label{fig:generic100}
    \end{subfigure}
    \caption{Specific and generic accuracy of multiple update attacks on FMNIST with $\nup=100$, using loss difference. As in Section~\ref{sec:one_update}, \sgdnew\ attacks perform worse relative to \sgdfull.}
\label{fig:multi100}
\end{figure}

\begin{figure}
    \centering
    \begin{subfigure}{0.4\textwidth}
    \includegraphics[width=\linewidth]{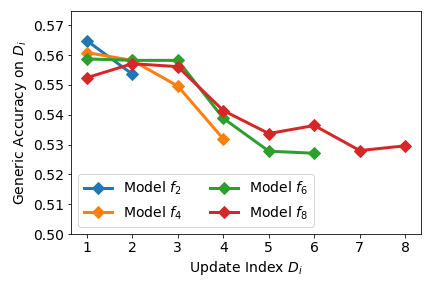}
    \caption{\sgdfull}
    \label{fig:generic_each_full}
    \end{subfigure}
    \begin{subfigure}{0.4\textwidth}
    \includegraphics[width=\linewidth]{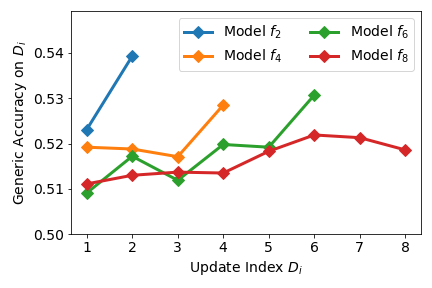}
    \caption{\sgdnew}
    \label{fig:generic_each_only}
    \end{subfigure}
\caption{Generic accuracy for each update index used to train a given model $f_i$. Results shown for the Back-Front attack with loss difference on FMNIST with \sgdfull and \sgdnew at $\nup=100$.}
\label{fig:generic_each}
\end{figure}

\paragraph{Section~\ref{sec:eval:q1} Experiments - Single Update.}
We report the precision and recall for \sgdnew\ attacks in Table~\ref{tab:precision_only}. These precision values are even higher than precisions reached by SGD-Full presented in Table~\ref{tab:single_update_precision}.


\paragraph{Section~\ref{sec:eval:q4} Experiments - Multiple Updates.}
We present in Figures \ref{fig:specific100} and \ref{fig:generic100} the results for attacks on FMNIST when $\nup=100$, to compare against the results from Section~\ref{sec:mult_update}, which use $\nup=10$.

In Figure~\ref{fig:specific100}, we show that, at larger update set sizes, attacks on SGD-Full can outperform attacks on \sgdnew, in line with experiments from Section~\ref{sec:one_update}. The specific accuracy of attacks with $\nup=100$ is also smaller than for attacks with $\nup=10$, which also corroborates our experiments in Section~\ref{sec:one_update}.

In Figure~\ref{fig:generic100}, we make similar observations to those in Figure~\ref{fig:specific100}. Attacks are less powerful at large update set sizes, and attacks on SGD-Full perform better relative to attacks on \sgdnew.